\PassOptionsToPackage{pagebackref,breaklinks,colorlinks,allcolors=cvprblue}{hyperref}
\documentclass[10pt,twocolumn,letterpaper]{article}

\usepackage{cvpr}              
\usepackage[utf8]{inputenc} 
\usepackage[T1]{fontenc}    
\usepackage{wrapfig}
\usepackage{hyperref}       
\usepackage{url}            
\usepackage{booktabs}       
\usepackage{amsfonts}       
\usepackage{nicefrac}       
\usepackage{microtype}      
\usepackage{xcolor}         
\usepackage{graphicx}
\usepackage{float}
\usepackage{subcaption}
\usepackage{enumitem} 
\usepackage{amsmath,amsthm}

\theoremstyle{plain}
\newtheorem{theorem}{Theorem}
\newtheorem{lemma}[theorem]{Lemma} 
\newtheorem{assumption}{Assumption}

\usepackage{bm}
\usepackage{algorithm}
\usepackage{algpseudocode}
\definecolor{cvprblue}{rgb}{0.21,0.49,0.74}

\def\confName{CVPR}
\def\confYear{2026}

\title{Personalized Longitudinal Medical Report Generation via Temporally-Aware Federated Adaptation}

\author{He Zhu, Ren Togo, Takahiro Ogawa, Kenji Hirata, Minghui Tang, Takaaki Yoshimura, \\
Hiroyuki Sugimori, Noriko Nishioka, Yukie Shimizu, Kohsuke Kudo, Miki Haseyama\\
Hokkaido University
}

\begin{document}
\maketitle
\begin{abstract}
Longitudinal medical report generation is clinically important yet remains challenging due to strict privacy constraints and the evolving nature of disease progression. Although federated learning (FL) enables collaborative training without data sharing, existing FL methods largely overlook longitudinal dynamics by assuming stationary client distributions, making them unable to model temporal shifts across visits or patient-specific heterogeneity—ultimately leading to unstable optimization and suboptimal report generation.
We introduce Federated Temporal Adaptation (FTA), a federated setting that explicitly accounts for the temporal evolution of client data. Building upon this setting, we propose FedTAR, a framework that integrates demographic-driven personalization with time-aware global aggregation. FedTAR generates lightweight LoRA adapters from demographic embeddings and performs temporal residual aggregation, where updates from different visits are weighted by a meta-learned temporal policy optimized via first-order MAML.
Experiments on J-MID (1M exams) and MIMIC-CXR demonstrate consistent improvements in linguistic accuracy, temporal coherence, and cross-site generalization, establishing FedTAR as a robust and privacy-preserving paradigm for federated longitudinal modeling.
\end{abstract}    
\section{Introduction}
Automatic medical report generation~\cite{kisilev2015medical, messina2022survey}, particularly from longitudinal chest CT scans that capture disease progression~\cite{newell2004report}, is central to reliable clinical diagnosis~\cite{hosny2018artificial}. By summarizing sequential imaging evidence and modeling temporal patterns across visits, it alleviates the manual reporting burden and improves diagnostic consistency. However, constructing temporally coherent and generalizable models remains challenging due to strict privacy regulations and fragmented institutional data silos~\cite{rajpurkar2022ai, chiruvella2021ethical}. Crucially, these constraints do not only limit the amount of available data, but also break the assumption that a patient's studies can be viewed as i.i.d. samples from a single static distribution. These barriers exacerbate inter-site and temporal heterogeneity, calling for a federated and time-aware setting capable of learning from distributed longitudinal cohorts without compromising privacy.

\begin{figure}[t]
  \centering
    \begin{subfigure}[b]{0.5\linewidth}
    \includegraphics[width=\linewidth]{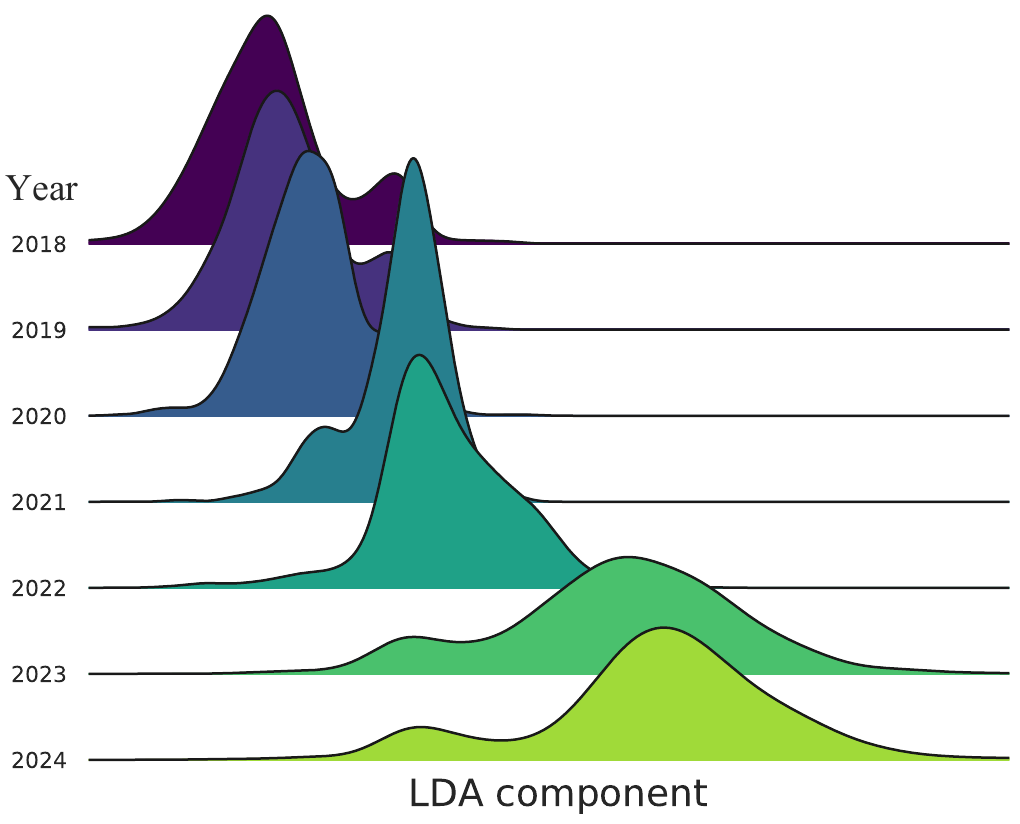}
    \caption{}
    \label{fig:intro_sub2}
  \end{subfigure}
  \hfill
  \begin{subfigure}[b]{0.4\linewidth}
    \includegraphics[width=\linewidth]{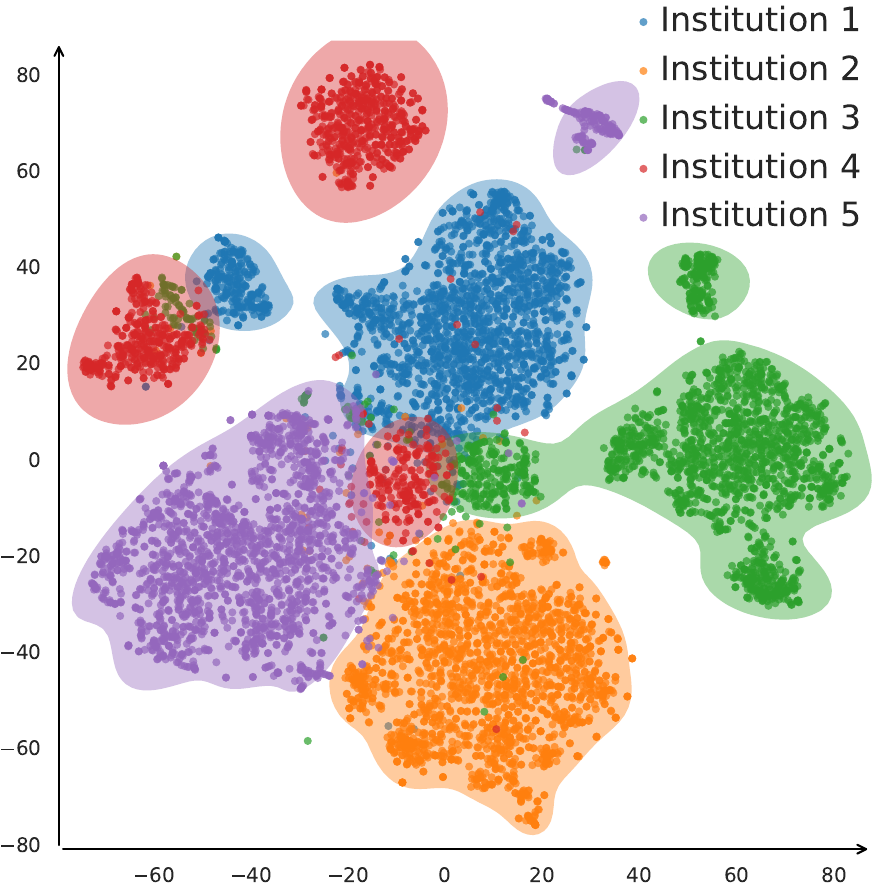}
    \caption{}
    \label{fig:intro_sub1}
  \end{subfigure}
  \caption{Visualization of semantic structure evolution in medical reports. All data are derived from the J-MID dataset. (a) Ridgeline kernel-density plots of LDA-derived topic components for annual reports (2018–2024) from one institution. (b) T-SNE embedding of TF-IDF vectors for reports from five institutions. }
  \label{fig:intro}
  \vspace{-2em}
\end{figure}

Federated Learning (FL) has emerged as a central paradigm for privacy-preserving collaboration across medical institutions~\cite{FedAvg, li2021survey, adnan2022federated, guan2024federated}, yet most formulations remain temporally static, modeling each client as a fixed data distribution. Recent FL approaches to medical report generation, including FedMRG~\cite{metmer2025fedmrg}, FedMME~\cite{wang2025multi}, and Chen et al.~\cite{chen2024medical}, follow this setting and treat all examinations on a client as if they were drawn from a single time-invariant distribution, thus ignoring longitudinal disease trajectories and patient-specific evolution. In practice, however, temporal and cross-site drift are pronounced: in our longitudinal chest CT cohort, the semantic distributions of radiology reports shift markedly across years, and study embeddings form well-separated clusters across institutions (Fig.~\ref{fig:intro}). Under such temporal heterogeneity, conventional ``one-size-fits-all'' aggregation strategies implicitly assume stationarity and tend to average out progression-related signals that are crucial for clinical interpretation. In other words, existing FL formulations are client-distribution–aware but essentially time-agnostic, leaving the longitudinal nature of medical report generation fundamentally under-modeled and calling for a federated setting that performs personalized, time-aware adaptation to evolving patient states and cross-institutional diversity.

We formalize Federated Temporal Adaptation (FTA), a novel FL problem setting that enables time-aware and personalized modeling under evolving data distributions and privacy constraints. FTA addresses temporal non-stationarity across sequential observations and entity-specific heterogeneity, challenges pervasive in longitudinal, streaming, and cross-institutional domains. To instantiate this setting, we propose FedTAR, a unified framework that integrates demographic-conditioned personalization and meta-learned temporal aggregation for parameter-efficient and theoretically stable optimization.
At the personalization level, each client encodes demographic information into compact statistical representations that drive a lightweight hypernetwork to produce patient-specific low-rank adapters. This mechanism provides a parameter-efficient way to capture who varies across federated participants while preserving privacy by avoiding the exchange of raw demographics.
At the global level, FedTAR performs temporal residual aggregation, learning visit-dependent weights that balance recent and historical updates. This yields a transferable temporal prior that captures when variation occurs, supporting stable optimization under evolving client distributions.
Together, these personalized adapters and temporal aggregation form the first concrete instantiation of FTA, jointly modeling who varies and when change occurs to achieve temporally coherent and privacy-preserving federated report generation.

We use the Japan Medical Image Database (J-MID)\footnote{\url{https://www.radiology.jp/j-mid/}}\,\cite{fujita2024advancing}, which contains $\approx1$ M CT/MR exams ($\approx300$ M DICOM slices) from more than ten hospitals.
The dataset links repeated scans of each patient, providing a unique large-scale resource for evaluating both cross-site federated and longitudinal models.
%
We also validated the framework on public datasets MIMIC-CXR~\cite{johnson2019mimic}, demonstrating its robustness across diverse clinical domains.

Our primary contributions can be summarized as follows:
\begin{itemize}
\item We introduce Federated Temporal Adaptation (FTA), a new FL setting that makes temporal evolution a first-class component of federated optimization.
\item We present FedTAR, the first instantiation of FTA, unifying personalized low-rank adapters with temporal residual aggregation for stable and efficient learning under evolving client distributions.
\item Experiments on J-MID and MIMIC-CXR demonstrate improved report quality, temporal coherence, and cross-site generalization, supported by convergence analysis.
\end{itemize}
\section{Related Work}
\subsection{Federated Learning for Medical Applications}
Federated Learning (FL)~\cite{guan2024federated, brisimi2018federated, sheller2020federated} enables multi-institutional model training without data sharing~\cite{lu2024federated}. Beyond early classification tasks, recent studies explored generative objectives such as medical report synthesis.
Chen and Pan~\cite{chen2024federated} combined FedAvg with evidence-driven optimization to improve global modeling, while FedMRG~\cite{metmer2025fedmrg} mitigated inter-client heterogeneity but assumed static data.
FedMME~\cite{wang2025fedmme} exploited large vision–language models for single-round FL, focusing on non-Independent and Identically Distributed (non-IID) robustness rather than temporal evolution.
Parameter-efficient personalization methods like PerAda~\cite{xie2024perada} and pFedLoRA~\cite{yi2024pfedlora} tailor local adapters yet remain time-invariant.
Collectively, these works advance multimodal federated generation but overlook temporal non-stationarity and longitudinal patient dynamics—the core focus of our Federated Temporal Adaptation setting.
\begin{figure}[t]
\centering
  \begin{subfigure}[t]{\linewidth}
    \centering
    \includegraphics[width=\linewidth]{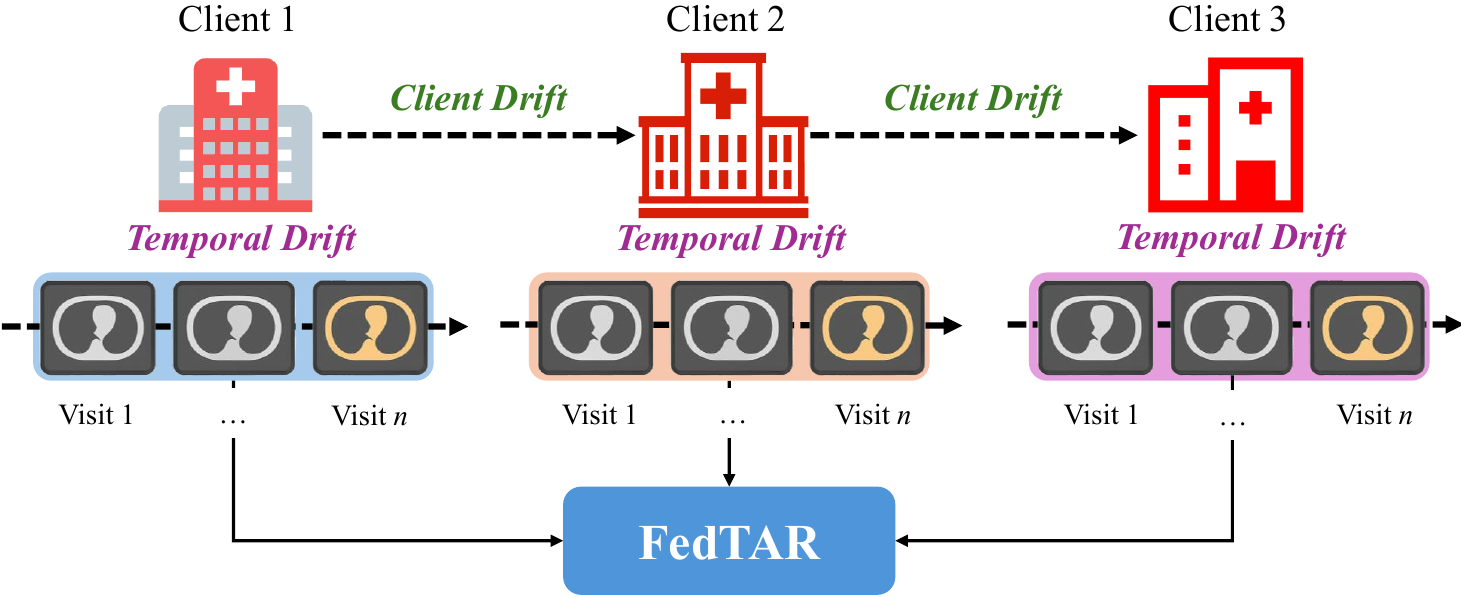}
    \caption{}
  \end{subfigure}
  \hfill
    \begin{subfigure}[t]{0.8\linewidth}
    \centering
    \includegraphics[width=\linewidth]{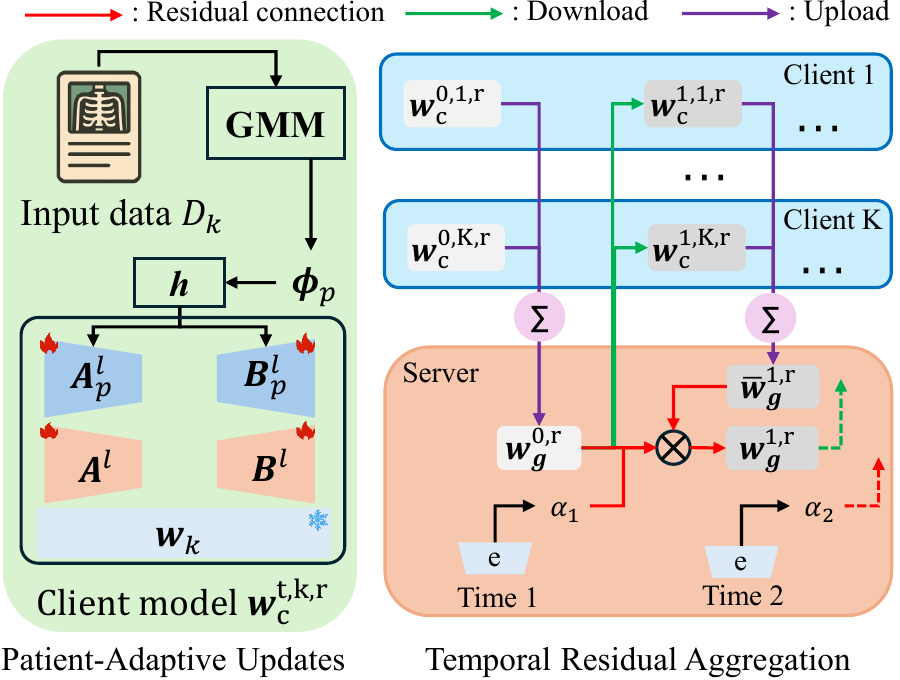}
    \caption{}
  \end{subfigure}
  \caption{Overview of the problem and our approach. (a) Client- and temporal-drift challenges in federated longitudinal medical report generation. (b) The proposed FedTAR framework.}
  \label{fig:overview}
\end{figure}
\subsection{Temporal Longitudinal Medical Imaging}
Longitudinal imaging is vital for tracking disease progression and guiding diagnosis~\cite{gerig2016longitudinal, jin2021predicting}, yet early report generation models overlooked temporal dependencies by treating each scan independently.
Hamamci et al.~\cite{hamamci2024ct2rep} introduced CT2Rep and CT2RepLong, which employ memory-driven decoders and cross-attention to fuse multi-timepoint CTs and prior reports, improving temporal coherence and diagnostic fidelity.
Though centralized, these works motivate privacy-preserving longitudinal modeling in federated settings.
In parallel, Ghari and Shen~\cite{ghari2024streamfed} enhanced streaming robustness via checkpoint ensembling to counter concept drift~\cite{gama2014survey}, while Chen et al.~\cite{chen2022timefed} proposed time-weighted aggregation for evolving data.
Together, these advances highlight the need for temporal fusion~\cite{mei2024medical} and dynamic optimization~\cite{atteia2023adaptive}, whose integration into federated medical report generation remains largely unexplored.
\section{Methodology} \label{section:method}
\subsection{FTA Problem Definition}
We consider Federated Temporal Adaptation (FTA) as a federated learning setting where each client holds sequential data collected over multiple time steps.
Formally, client $k \in \{1,\dots,K\}$ possesses a local dataset $\mathcal{D}_k = \{p_n\}_{n=1}^{n_k}$, and each sequence
\[
p_n = \{x_{k,t}, y_{k,t}\}_{t=1}^{T}, (x_{k,t}, y_{k,t})\sim P_{k,t},
\]
contains the input $x_{k,t}$ observed at time step $t$ and the corresponding target $y_{k,t}$ drawn from a time-varying client-specific distribution $P_{k,t}$.
The global objective is to learn a model that performs well across both clients and time, typically written as
\begin{equation}
\min_{\bm{w}}
\sum_{k=1}^{K} \sum_{t=1}^{T}
\mathcal{L}\bigl(f(\bm{w}; x_{k,t}),\, y_{k,t}\bigr), 
\end{equation}
where $\bm{w}$ denotes the global model parameters, $f(\cdot;\bm{w})$ is the prediction model, and $\mathcal{L}$ is a task-specific loss function.

Compared with standard FL, which assumes each client $k$ draws i.i.d. samples $(x,y)\sim P_k$ from a stationary distribution, and online FL, which receives streaming mini-batches $\mathcal{B}_{k,\tau}\sim Q_{k,\tau}$ over rounds $\tau$ where the time index is attached to the optimization process rather than the data distribution, FTA explicitly models client-wise temporal distributions $\{P_{k,t}\}_{t=1}^{T}$ with sequences $\{(x_{k,t},y_{k,t})\}_{t=1}^{T}$, and thus making temporal drift a first-class component of the federated objective. 
The learning objective is therefore to obtain a global model that adapts coherently across time while remaining consistent across clients under privacy constraints. A more formal comparison between FTA and other FL problem definitions is provided in the Appendix.
In this work, we instantiate this general formulation in the context of longitudinal medical report generation, where sequential imaging and textual records are distributed across institutions.
\subsection{Preliminary}
An overview of our method is shown in Fig.~\ref{fig:overview}. We propose a novel multimodal FL framework consists of two sequentially coupled modules: (1) Client-Side Patient-Adaptive Updates, where we encode low-dimensional patient data via GMMs and dynamically generate personalized LoRA weights for local fine-tuning; and (2) Temporal Residual Aggregation, which computes residuals at each temporal step and employs hypergradient-based meta-learning to adaptively weight and aggregate client updates. We further integrate a hierarchical communication protocol to minimize overhead. Consequently, we deliver both theoretical soundness and practical efficiency for longitudinal medical report generation.
\subsection{Client-Side Patient-Adaptive Updates}
To personalize local adaptation, we condition LoRA weight generation on patient-specific data. Each client $k$ receives global parameters $w$ and its patient demographics $\mathrm{info}_k=\{\mathrm{id}_p,\mathrm{age}_p,\mathrm{sex}_p\}_{p=1}^{n_k}$.
We compute a normalized profile vector $\bm{v}_p\in\mathbb{R}^3$ for each patient $p$ as:
\begin{equation}
    \bm{v}_p = \bigl[\operatorname{Norm}(\mathrm{hash}(\mathrm{id}_p)), \operatorname{Norm}(\mathrm{age}_p), \mathrm{enc}(\mathrm{sex}_p)\bigr]^\top.
\end{equation}
Then, we encode $\bm{v}_p$ into a soft cluster assignment $\mathbf{q}_p$ via a GMM as $\mathbf{q}_p = \mathrm{GMM}(\bm{v}_p)\in\Delta^{n_{comp}-1}$. The soft assignment captures the probabilistic membership of $\bm{v}_p$ across $n_{comp}$ latent clusters, allowing for smooth interpolation between different patient subgroups. Subsequently, the probability vector $\mathbf{q}_p$ is linearly projected into a $d$-dimensional embedding space through a learnable transformation, yielding the patient embedding $\bm{\phi}_p = \mathbf{W}_{\mathrm{proj}}\mathbf{q}_p + \mathbf{b}_{\mathrm{proj}}\in\mathbb{R}^d$. This projection enables the model to adaptively reweight cluster contributions based on downstream tasks.

Thereafter, we use each patient embedding $\bm{\phi}_p$ to modulate a low-rank adapter. Let $W^l_{\mathrm{client}}\in\mathbb{R}^{e\times e}$ be the $l$-th client model weight block, and suppose our client LoRA adapter is
\[
\Delta W^l_{\rm client} \;=\; A^l B^{l\top},
\quad A^l\in\mathbb{R}^{e\times r},\;B^l\in\mathbb{R}^{r\times e}.
\]
We introduce a lightweight hypernetwork $h^l: \mathbb{R}^d \;\to\; \mathbb{R}^{(e\times r)+(r\times e)}$
that maps $\bm{\phi}_p$ to patient-specific adapter parameters as
\begin{equation}
    \bigl[A_p^l(\bm{\phi}_p),\,B_p^l(\bm{\phi}_p)\bigr] \;=\; h^l(\bm{\phi}_p),
\end{equation}
and define the total adapted weight for patient $p$ at layer $l$ as follows:
\begin{equation}
W^l_p \;=\; 
W^l_{\mathrm{client}} 
\;+\;\underbrace{A^l B^{l\top}}_{\Delta W^l_{\rm client}}
\;+\;\underbrace{A_p^l(\bm{\phi}_p)\,B_p^l(\bm{\phi}_p)^{\top}}_{\Delta W^l_p}.
\end{equation}
At each forward pass for client $k$, we retrieve its demographics set $\{p\in D_k\}$, compute each $\bm{\phi}_p$, instantiate $\{\Delta W^l_p\}$ via the hypernetworks, and apply $W^l_p$ in place of $W^l_{\mathrm{client}}$.
Training proceeds by minimizing the standard clinical loss
\begin{equation}
\mathcal{L}\bigl(\{W^l_{\mathrm{client}}\},\{A^l,B^l\},\{h^l\}\bigr)
\;=\;
\sum_{(I,R)\in\mathcal{D}_k}
\ell\bigl(f(I;W_p),\,R\bigr),
\end{equation}
where gradients flow into the client model weights $W^l_{\mathrm{client}}$, the shared adapter factors $\{A^l,B^l\}$, the hypernetwork parameters $\{h^l\}$, and the projection of $\bm{\phi}_p$.

Overall, our client-side patient-adaptive updates integrate demographic-conditioned soft clustering with hypernetwork-based LoRA, enabling fine-grained personalization with minimal communication overhead. This design allows smooth interpolation across latent subpopulations and improves generalization under heterogeneous clinical data (Algorithm~\ref{alg:patient_adaptive_updates}).
\begin{algorithm}[H]
\caption{Client‐Side Patient‐Adaptive Updates}
\label{alg:patient_adaptive_updates}
\begin{algorithmic}[1]
\Require LoRA adapter $\{A^l,B^l\}_{l=1}^L$, hypernetworks $\{h^l\}_{l=1}^L$, projection params $(\mathbf{W}_{\mathrm{proj}},\mathbf{b}_{\mathrm{proj}})$, local data ${(I,R)}$
\ForAll{patients $p \in D_k$}
  \State $\displaystyle \bm v_p \leftarrow \bigl[\mathrm{Norm}(\mathrm{hash}(\mathrm{id_p})),\;\mathrm{Norm}(\mathrm{age_p}),\;\mathrm{enc}(\mathrm{sex_p})\bigr]^\top$
  \State $\displaystyle \mathbf q_p \leftarrow \mathrm{GMM}(\bm v_p)$; $\displaystyle \bm{\phi}_p \leftarrow \mathbf{W}_{\mathrm{proj}}\,\mathbf q_p + \mathbf{b}_{\mathrm{proj}}$
    \For{$l=1$ to $L$}
    \State $(A_p^l, B_p^l) \leftarrow h^l(\bm{\phi}_p)$
      \State $\Delta W^l_{\rm client} \leftarrow A^l\,B^{l\top}$; $\Delta W^l_p \leftarrow A_p^l\,B_p^{l\top}$
    \State $W^l_p \leftarrow W^l_{\mathrm{client}} + \Delta W^l_{\rm client} + \Delta W^l_p$
  \EndFor
\EndFor
\State Compute $\displaystyle \mathcal{L} \;=\;\sum_{(I_p,R_p)\in D_k} \ell\bigl(f((I_p;\{W_p^l\}),R_p)\bigr)$
\State Update $\{A^l,B^l\},\{h^l\},\mathbf{W}_{\mathrm{proj}},\mathbf{b}_{\mathrm{proj}}$ via backprop
\end{algorithmic}
\end{algorithm}
\subsection{Temporal Residual Aggregation}
Longitudinal clients continuously accumulate new visits, causing their optimization trajectories to drift in both direction and speed; because static parameter averaging cannot anticipate or align these future temporal shifts, we aggregate residual updates, and use meta-learning to adapt their weights so the global model can predictively adjust to fast-changing, temporally expanding data.
\subsubsection{Two Layer Temporal Aggregation}  
At communication round $r$, each client $k$ sequentially uploads its locally fine-tuned models  
$\{\,\bm{w}_c^{t, k, r}\}_{t=0}^T$ after training on time-point $t$ data.  
At time step $t$, we perform a single round of local weighted model aggregation to obtain the base global model weights $\bar{\bm{w}}_g^{(t,r)}$:
\begin{equation}
  \bar{\bm{w}}_g^{(t,r)} = \sum_{k=1}^K \frac{1}{n_k} \bm{w}_c^{t, k, r}\,.
  \label{eq:time_avg}
\end{equation}
Subsequently, rather than discarding previous temporal information, we incorporate a residual correction as follows:
\begin{equation}
      \bm{w}_g^{(t,r)} = \bm{w}_g^{(t-1,r)} + \alpha_t \bigl(\bar{\bm{w}}_g^{(t,r)} - \bm{w}_g^{(t-1,r)}\bigr), \label{eq:residual_update}
\end{equation}
where we initialize \(\bm{w}_g^{(0,r)} = \bm{w}_g^{(T,r-1)}\) for \(r>1\) and \(\bm{w}_g^{(0,1)}\) is set to the initial global model. Thus, at each time step, the new global model is obtained by adding a weighted residual that reflects the update magnitude at that time point.

The per-step coefficient $\alpha_t\in[0,1]$ will be learned via a meta-learning procedure.  
However, even before specifying a particular schedule for $\{\alpha_t\}$, the following theorems reveal intrinsic convergence and stability properties of the residual formulation, thereby justifying the subsequent meta-learning of $\alpha_t$. 

\begin{theorem} \label{thm:convex_hull}
Let $\Delta_t = \bar{\bm w}_g^{(t,r)} - \bm w_g^{(t-1,r)}$, $\beta_x^{(t)} = \alpha_x \prod_{y=x+1}^t (1 - \alpha_y)\quad (1\le x\le t)$, and  $\beta_0^{(t)} = \prod_{y=1}^t (1 - \alpha_y).$
Then the iterates defined by \eqref{eq:residual_update} admit the closed form
\[
\bm w_g^{(t,r)}
= \beta_0^{(t)}\,{\bm w}_g^{(0,r)} + \sum_{x=1}^t \beta_x^{(t)}\,\bar{\bm w}_g^{(x,r)},
\]
with $\beta_x^{(t)}\ge0$ and $\sum_{x=0}^t\beta_x^{(t)}=1$.  
Consequently, $\bm w_g^{(t,r)}$ always lies in the convex hull of $\{\bar{\bm w}_g^{(1,r)},\dots,\bar{\bm w}_g^{(t,r)}\}$, and if $\lim_{t\to\infty}\bar{\bm w}_g^{(t,r)}=w^*$ then $\lim_{t\to\infty}\bm w_g^{(t,r)}=w^*$.
\end{theorem}
This result guarantees that, irrespective of how the meta-learner chooses $\alpha_t$, the global model retains a smooth memory of all past snapshots by forming a convex combination of $\{\bar{\bm w}_g^{(1,r)},\dots,\bar{\bm w}_g^{(t,r)}\}$, where the pre-trained initial model $\bm w_g^{(0,r)}$ serves as an anchoring vertex in the convex hull, thus imbuing the aggregation process with domain-specific priors that are smoothly incorporated and subsequently attenuated according to the learned $\{\alpha_t\}$. When generating medical reports from a patient’s longitudinal examination records—our aggregation automatically balances contributions from each time point, ensuring that early biomarkers and recent observations are both reflected in a coherent, temporally aware global model. A complete analysis is deferred to the Appendix A.1. 

\begin{theorem}\label{thm:bounded_update}
Assume that for all $t$, $\|\Delta_t\|\le G$ for some constant $G>0$, and that $0\le\alpha_t\le1$.  Then the update magnitude at each time step satisfies
\[
\bigl\|\bm w_g^{(t,r)} - \bm w_g^{(t-1,r)}\bigr\|
= \alpha_t\,\|\Delta_t\|
\;\le\;\alpha_t\,G.
\]
In particular, if $\alpha_t\to0$ as $t\to\infty$, then the per-step update norm vanishes, ensuring smoother convergence and preventing excessive oscillations.
\end{theorem}
This boundedness property underpins the trade-off between adaptivity and stability: large $\alpha_t$ allows rapid response to new data, whereas small $\alpha_t$ dampens noise and enforces convergence.  Consequently, meta-learning $\{\alpha_t\}$ enables an automated balance between optimization speed and robust stability in non-stationary federated settings. A complete analysis is deferred to the Appendix A.2.
\subsubsection{Meta-Learned Coefficients}
\begin{algorithm}[t]
\caption{Temporal Residual Aggregation}
\label{alg:global_temporal_aggregation}
\begin{algorithmic}[1]
\Require Initial global model $\bm{w}_g^{(0,0)}$, meta‐parameters $\psi$, total rounds $R$, time points $T$
\For{$r = 1$ to $R$}
  \State $\bm{w}_g^{(0,r)} \gets \bm{w}_g^{(T,\,r-1)}$
  \For{$t = 1$ to $T$}
    \For{each client $k=1,\dots,K$ in parallel}
      \State $\bm{w}_c^{t, k, r} \gets \mathtt{ClientUpdates}(\bm{w}_g^{(t-1,r)},D_k)$
    \EndFor
    \Statex\quad$\bar{\bm{w}}_g^{(t,r)} \gets \sum_{k=1}^K \frac{1}{n_k}\,\bm{w}_c^{t, k, r}$
    \Statex\quad$\bm{w}_g^{(t,r)} \gets \bm{w}_g^{(t-1,r)} + \alpha_t(\psi)\,(\bar{\bm{w}}_g^{(t,r)} - \bm{w}_g^{(t-1,r)})$
  \EndFor
  \State $\psi \gets \psi - \eta\,\nabla_\psi\,\mathcal{L}_{\mathrm{val}}(\bm{w}_g^{(T,r)}(\psi))$
\EndFor
\end{algorithmic}
\end{algorithm}
Building on the convex‐combination and bounded‐update properties established in Theorems~\ref{thm:convex_hull} and~\ref{thm:bounded_update}, we introduce a data‐driven procedure to calibrate the temporal weights $\{\alpha_t\}$ so as to optimize both responsiveness to recent updates and overall convergence stability. Specifically, we parameterize $\alpha_t$ as
\begin{equation}
    u_t = g\bigl(e(t);\psi\bigr), 
\qquad
[\alpha_1,\dots,\alpha_T] = \mathrm{Softmax}([u_1,\dots,u_T]),
\end{equation}
where $g(\cdot;\psi)$ denotes a lightweight multilayer perceptron mapping time embeddings $e(t)$ to unnormalized scores. This construction ensures that $\alpha_t\in(0,1)$ and that the resulting residual aggregation remains within the convex hull of past snapshots, thereby preserving the theoretical guarantees of smooth convergence.

However, selecting $\psi$ directly to minimize training loss can lead to overfitting temporal fluctuations; therefore, we adopt a bilevel optimization strategy. In the inner loop, temporal aggregation across $t=1,\dots,T$ yields the aggregated model $\bm{w}_g^{(T,r)}(\psi)$ according to \eqref{eq:residual_update}. In the outer loop, we update the meta‐parameters by minimizing a held‐out validation loss:
\begin{equation}
    \psi \leftarrow \psi - \eta\,\nabla_\psi\,\mathcal{L}_{\mathrm{val}}\bigl(\bm{w}_g^{(T,r)},\psi\bigr),
\end{equation}
where $\mathcal{L}_{\mathrm{val}}$ measures generalization performance on longitudinal medical report generation tasks. The detailed derivation of the hypergradient computation is provided in the Appendix A.3. By leveraging hypergradient descent, this meta‐learning scheme systematically refines $\{\alpha_t\}$ to balance optimization speed against robust stability across temporal domains. We summarize this process in Algorithm~\ref{alg:global_temporal_aggregation}.

Moreover, the softmax parameterization coupled with hypergradient-based adaptation circumvents the need for manual scheduling of residual coefficients, thereby automating the trade-off between rapid adaptation to non-stationary data and dampening of noisy updates. This approach ensures that large coefficients are assigned to time points with significant distributional shifts, whereas minor or spurious fluctuations are attenuated, leading to improved convergence properties and better generalization under temporal heterogeneity. We note that, under standard smoothness and Lipschitz continuity assumptions on the meta‐learner, the bilevel optimization procedure admits convergence guarantees; a complete analysis is deferred to the Appendix A.4.

\section{Experiments} \label{section:exp}
\begin{table*}[t]
\centering
\small
\caption{Comparison of our proposed method (FedTAR) with standard FL baselines on the chest CT report generation task. Evaluation metrics include BLEU-$n$ ($n$=1–4), ROUGE-L, and CIDEr, which collectively reflect linguistic precision, recall, and content relevance. FedTAR achieves consistent improvements, demonstrating the effectiveness of demographics- and temporally-aware adaptation.}
\label{tab:main_results}
\begin{tabular}{l|cccccc}
\toprule
         & BLEU-1 & BLEU-2 & BLEU-3 & BLEU-4 & ROUGE-L & CIDEr \\ \midrule
FedAvg~\cite{FedAvg}   & 38.40 & 24.68 & 16.16 & 10.98 & 28.54 & 31.70 \\
FedProx~\cite{FedProx}  & 38.32 & 24.48 & 15.95 & 11.00 & 28.58 & 31.62 \\
SCAFFOLD~\cite{karimireddy2020scaffold} & 35.40 & 22.38 & 14.80 & 10.10 & 27.12 & 24.82 \\
FedAdam~\cite{FedAdam}  & 38.26 & 24.54 & 16.05 & 10.93 & 28.61 & 30.62 \\
FedYogi~\cite{FedAdam}  & 35.39 & 22.37 & 13.72 & 10.46 & 27.91 & 24.79 \\
DRFA~\cite{deng2020distributionally}     & 36.80 & 24.02 & 16.63 & 11.60 & 28.75 & 29.51 \\ \midrule
FedTAR (ours)       & \textbf{40.08}  & \textbf{25.94} & \textbf{17.80}  &\textbf{12.40} & \textbf{29.54} & \textbf{42.80} \\ \bottomrule
\end{tabular}
\end{table*}
\begin{table*}[t]
\centering
\small
\caption{Ablation study evaluating the contributions of the Gaussian Mixture Model (GMM)-based demographics embedding and the temporal residual aggregation module. Removing either component leads to consistent performance degradation across all metrics, confirming their necessity for capturing patient-specific heterogeneity and longitudinal dynamics.}
\label{tab:ablation_results}
\begin{tabular}{l|cccccc}
\toprule
         & BLEU-1 & BLEU-2 & BLEU-3 & BLEU-4 & ROUGE-L & CIDEr \\ \midrule
FedTAR w/o GMM   & \textbf{40.38} & 25.08 & 17.59 & 12.07  & 28.65 & 41.56 \\
FedTAR w/o temporal  & 38.83 & 25.82 & 17.28 & 11.21 & 29.50 & 41.88 \\\midrule
FedTAR    &  40.08  & \textbf{25.94} &\textbf{17.80}  &\textbf{12.40} & \textbf{29.54} & \textbf{42.80} \\ \bottomrule
\end{tabular}
\end{table*}
\begin{table*}[t]
\centering
\small
\caption{Quantitative evaluation on the MIMIC-CXR dataset. Higher values indicate better performance.}
\label{tab:mimic}
\begin{tabular}{lccccccccc}
\toprule
 & CE-P & CE-R & CE-F1 & BLEU-1 & BLEU-2 & BLEU-3 & BLEU-4 & ROUGE-L & CIDEr \\
\midrule
FedAvg~\cite{FedAvg}   & 68.58 & 22.63 & 28.43 & 29.16 & 19.49 & 14.36 & 11.00 & 35.77 & 37.54 \\
FedProx~\cite{FedProx}   & \textbf{69.83} & 24.59 & 30.26 & 27.16 & 17.80 & 12.85 & 9.70 & 34.11 & 37.38 \\
SCAFFOLD~\cite{karimireddy2020scaffold} & 44.58 & 22.6 & 21.47 & 28.00 & 17.23 & 14.21 & 9.86 & 38.05 & 99.11 \\
FedAdam~\cite{FedAdam}  & 58.74 & 25.95 & 31.80 & \textbf{39.17} & \textbf{28.56} & 22.42 & 18.39 & 43.08 & 97.07 \\
FedYogi~\cite{FedAdam}  & 52.02 & 23.36 & 29.90 & 35.55 & 22.57 & 18.94 & 13.84 & 37.03 & 94.11 \\
DRFA~\cite{deng2020distributionally}     & 47.67 & 26.66 & 30.30 & 36.81 & 26.94 & 21.30 & 17.64 & 42.71 & 108.10 \\\midrule
FedTAR (ours)     & 52.73 & \textbf{28.82} & \textbf{33.21} & 36.59 & 27.83 & \textbf{22.60} & \textbf{19.08} & \textbf{44.42} & \textbf{120.58} \\
\bottomrule
\end{tabular}
\end{table*}
\begin{table*}[t]
\centering
\small
\caption{Comparison with centralized generation models on the J-MID dataset. Higher values are better.}
\label{tab:com_mg_model}
\begin{tabular}{lccccccccc}
\toprule
& CE-P & CE-R & CE-F1 & BLEU-1 & BLEU-2 & BLEU-3 & BLEU-4 & ROUGE-L & CIDEr \\
\midrule
CT2RepLong   & 13.37 & \textbf{13.12} & 13.24 & \textbf{47.20} & \textbf{28.46} & \textbf{24.26} & 10.57 & 35.63 & 25.68 \\
FedTAR  (ours)      & \textbf{17.65} & 12.76 & \textbf{14.47} & 40.08 & 25.94 & 20.54 & \textbf{12.14} & \textbf{29.54} & \textbf{42.80} \\
\bottomrule
\end{tabular}
\end{table*}
\begin{table*}[t]
\centering
\caption{Paired significance analysis between FedTAR and FedAdam on MIMIC-CXR. Significant results: 95\% CI excludes 0 and \(p<0.05\).}
\label{tab:significance_compact}
\resizebox{\linewidth}{!}{
\begin{tabular}{lccccccccc}
\toprule
Statistic & CE-P & \textbf{CE-R} & \textbf{CE-F1} & BLEU-1 & \textbf{BLEU-2} & \textbf{BLEU-3} & \textbf{BLEU-4} & \textbf{CIDEr} & \textbf{ROUGE-L} \\
\midrule
95\% CI & (-0.02,0.08) & (0.01,0.10) & (0.00,0.09) & (-0.00,0.02) & (0.01,0.03) & (0.01,0.03) & (0.01,0.03) & (0.17,0.57) & (0.01,0.03) \\
$p$-value & 0.25 & 0.02 & 0.04 & 0.09 & <0.001 & <0.001 & <0.001 & <0.001 & <0.001 \\
Win-rate (\%) & 88.98 & 98.84 & 97.76 & 95.71 & 100.00 & 100.00 & 100.00 & 100.00 & 100.00 \\
Significant & No & Yes & Yes & No & Yes & Yes & Yes & Yes & Yes \\
\bottomrule
\end{tabular}
}
\end{table*}
The objectives of our experimental study are to demonstrate that our temporally-aware federated adaptation framework can effectively capture longitudinal disease progression across five distinct clinical sites, to show that patient demographics–conditioned low-rank adaptation via Gaussian Mixture Modeling enhances the quality of generated chest CT reports compared to conventional FL baselines, and to verify that meta-learned temporal residual aggregation improves convergence stability and generalization performance under heterogeneous institutional data distributions.
\paragraph{Dataset Composition}
We conducted all experiments on a proprietary, longitudinal chest CT dataset newly assembled from five independent medical institutions, comprising 36,674, 5,967, 6,631, 6,688, and 1,439 patient cases, respectively. Each institution contributed de-identified CT volumes, corresponding radiological narratives authored by board-certified clinicians, and minimal patient demographics—namely a pseudonymized identifier, age, and sex. Crucially, every patient in our cohort underwent exactly five sequential CT examinations, thereby capturing fine-grained temporal progression of pulmonary findings. We also conducted evaluations on the public MIMIC-CXR dataset~\cite{johnson2019mimic}, which contains 370,000 medical images. 
\paragraph{Federated Client Partitioning}
To simulate a realistic FL environment under non-IID conditions, we partitioned the entire dataset by source institution into five clients. Each client holds all five time-ordered examinations for its local patient subpopulation, which enables us to evaluate both the convergence behavior and generalization performance of our patient-adaptive LoRA embeddings and meta-learned temporal aggregation under heterogeneous temporal dynamics. Moreover, by preserving each client’s unique distribution of disease trajectories, our setup rigorously tests the ability of Gaussian Mixture–conditioned embedding and hypernetwork-driven adaptation to reconcile local variability with global model consistency. For MIMIC-CXR dataset, we partitioned the data into four clients. Since patient-level identifiers are unavailable, we assigned ``subject id'' as input.
\paragraph{Ethics and Privacy}
All data collection and processing procedures received approval from the respective institutional review boards, and no direct or quasi-identifiers are retained at any stage. However, due to strict patient-privacy regulations, the dataset itself is not publicly released. This configuration provides a robust testbed for assessing optimization stability and personalized generalization in temporally aware federated medical report generation.
\subsection{Implementation Details}
Our method is implemented in PyTorch 2.4.1 and executed on NVIDIA RTX A6000 GPUs. The image encoder is a Convolutional vision Transformer warmed up with an ImageNet-21K checkpoint, and the text decoder is DistilGPT2~\cite{sanh2019distilbert} pretrained on clinical corpora. We attach LoRA adapters of rank $r=4$ with $\alpha=128$ to every transformer layer and optimize all parameters using AdamW with a base learning rate of $10^{-5}$ for model and adapter parameters and $10^{-4}$ for meta-learned temporal coefficients. We implement patient‐profile clustering by fitting a diagonal‐covariance GMM with sixteen components using scikit‐learn’s \texttt{GaussianMixture} class, configured with a small regularization on the diagonal entries to prevent singularities and a strict convergence tolerance for EM iterations. Each patient is represented by a normalized three‐dimensional vector—comprising a SHA‐256–based hashed identifier scaled to \([0,1]\), age divided by 100, and a binary encoding of sex. After fitting the GMM on the stacked matrix of these profile vectors, we obtain for each patient a soft assignment vector of length sixteen via \texttt{predict\_proba}. These mixture weights are then mapped into the model’s embedding space through a learnable linear projection layer, where only the projection’s weights and bias are updated during downstream federated training.

The temporal residual aggregation hypernetwork is a lightweight, fully differentiable module that embeds each integer time index via a learnable table, feeds the embeddings into a single-hidden-layer MLP with ReLU to produce unnormalized scores over all time steps, and applies a softmax to obtain aggregation coefficients that sum to one for the residual update.
\subsection{Results}
\paragraph{Comparison with Baselines}
We report BLEU-1 through BLEU-4~\cite{papineni2002bleu}, ROUGE-L~\cite{lin2004rouge}, and CIDEr~\cite{vedantam2015CIDEr} for evaluation~\cite{sloan2024automated}. We compare our proposed method (FedTAR) against six representative FL baselines: FedAvg~\cite{FedAvg}, FedProx~\cite{FedProx}, SCAFFOLD~\cite{karimireddy2020scaffold}, FedAdam~\cite{FedAdam}, FedYogi~\cite{FedAdam}, and DRFA~\cite{deng2020distributionally}. For all six federated baselines, we follow the conventional centralized-training protocol that pools data from all time points into a single dataset. The average scores across institutions are summarized in Table~\ref{tab:main_results}, and detailed per-site results are reported in Appendix~B.1. FedTAR consistently surpasses all baselines across BLEU, ROUGE-L, and CIDEr metrics, confirming its advantage in both linguistic precision and semantic consistency. These improvements stem from the framework’s ability to condition model updates on patient demographics and temporal progression, allowing the global model to maintain coherence across visits. In particular, FedTAR shows higher early n-gram accuracy and stronger structural alignment, indicating its capacity to capture subtle longitudinal evolution and site-specific variations. Qualitative report samples illustrating these gains are presented in Appendix~B.2.
\begin{figure}[t]
\centering
  \begin{subfigure}[t]{0.7\linewidth}
    \centering
    \includegraphics[width=\linewidth]{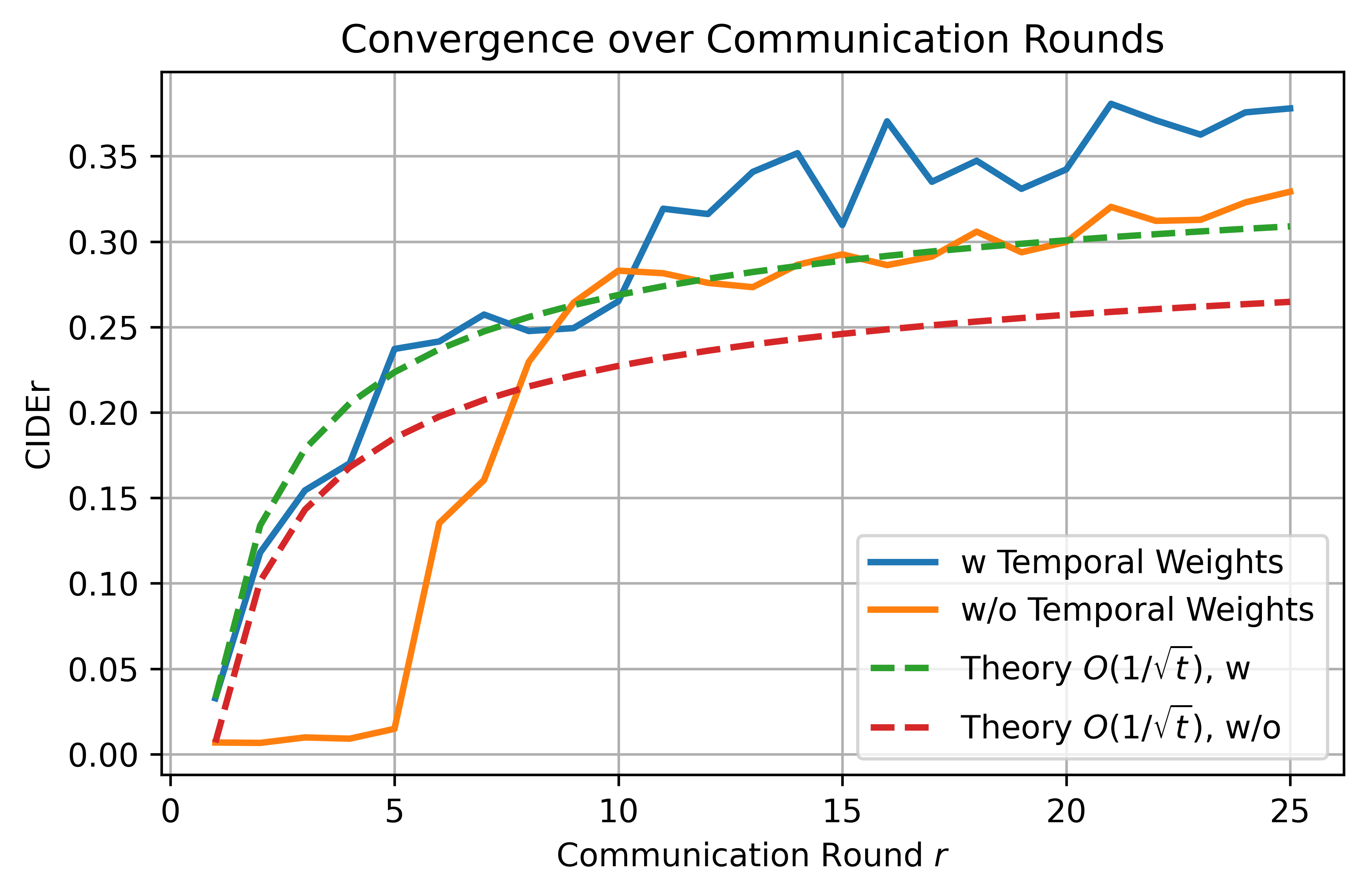}
    \caption{Empirical and theoretical convergence of CIDEr scores over communication rounds.}
    \label{fig:convergence}
  \end{subfigure}
  \hfill
    \begin{subfigure}[t]{0.7\linewidth}
    \centering
    \includegraphics[width=\linewidth]{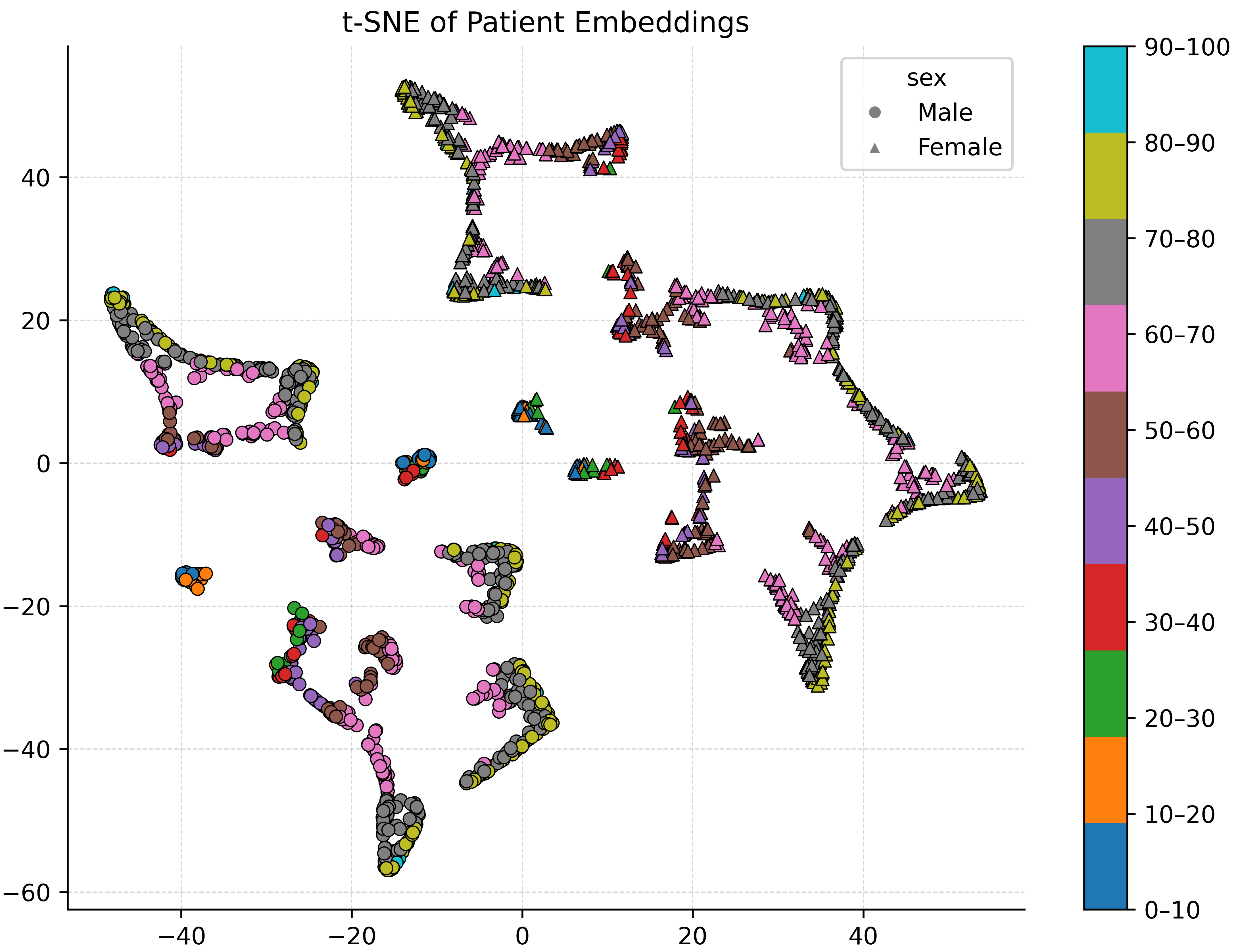}
    \caption{Visualization of patient embeddings $\bm{\phi}_p$.}
    \label{fig:patient_embeddings}
  \end{subfigure}
  \caption{Overall caption describing both subfigures.}
\end{figure}
\paragraph{Ablation Study} As shown in Table~\ref{tab:ablation_results}, we perform ablation experiments on the two core components of FedTAR: the GMM-based demographics embedding and the meta-learned temporal aggregation. Removing the GMM module markedly reduces BLEU and CIDEr scores, confirming that patient-conditioned embeddings are essential for capturing demographic and inter-patient variability. Disabling the temporal hypernetwork similarly lowers ROUGE-L and recall-oriented metrics, demonstrating that adaptive temporal weighting is crucial for modeling longitudinal disease progression and mitigating inter-site drift. These complementary effects verify that both personalization and temporal adaptation are indispensable to the overall performance. Extended hyperparameter sensitivity to the LoRA rank~$r$ and GMM component number~$k$ is reported in Appendix~B.3.
\paragraph{Computational and Communication Efficiency.}
We evaluate the overhead introduced by FedTAR. 
Adding LoRA adapters and the GMM-based demographics encoder increases per-layer FLOPs by less than 0.001\% on average, and GMM fitting is performed once per client, adding $<0.5\%$ total runtime. 
During communication, only lightweight LoRA parameters ($A,B$) and their gradients are exchanged—no patient-specific tensors or demographics leave the client—reducing payload size by over $10\times$ compared with full model synchronization. 
Overall, the proposed modules deliver substantial performance gains with negligible additional computational or communication cost.
\paragraph{Cross-dataset Evaluation}
We further validated the framework on the MIMIC-CXR dataset using the same settings as in the main study. Clinical Efficacy (CE) metrics~\cite{chexbert}, which measure clinical consistency between generated and reference reports, were employed alongside standard text metrics. As shown in Table~\ref{tab:mimic}, our method consistently outperforms prior approaches, achieving relative gains of +4.4\% in macro-F1, +8.1\% in macro-Recall, and +11.5\% in CIDEr, along with improvements in ROUGE-L and BLEU-4. These results demonstrate stronger semantic fidelity and clinical consistency in generated reports.
\paragraph{Comparison with Centralized Report Generation}
To evaluate cross-site generalization, we compared our federated framework with centralized report generation models CT2RepLong on the J-MID dataset. As shown in Table~\ref{tab:com_mg_model}, our method achieves higher macro-F1, CE precision, and a 66.7\% gain in CIDEr, indicating stronger semantic alignment and phase-level coherence. Despite a slight recall drop (2.7\%), the overall results demonstrate that the proposed temporal residual aggregation effectively preserves privacy while maintaining or surpassing centralized performance across key clinical and semantic metrics.
\paragraph{Statistical Significance Analysis}
We validate that the observed improvements are statistically reliable by comparing FedTAR with the strongest baseline FedAdam on the MIMIC-CXR dataset.  
For each study, we computed the metric difference and drew 5000 bootstrap replicates to estimate 95\% confidence intervals (CI).  
Paired permutation tests provided \(p\)-values, and the win-rate indicates the percentage of test studies where FedTAR outperforms FedAdam.  
A metric is significant when the 95\% CI excludes 0 and \(p < 0.05\).
\paragraph{Convergence Analysis}
Figure~\ref{fig:convergence} plots the empirical CIDEr trajectories of temporal residual aggregation, its unweighted variant, and the $O(\frac{1}{\sqrt{t}})$ reference.
FedTAR increases rapidly in the early rounds and saturates before the baseline shows noticeable gains.
Across training, the blue curve stays above the theoretical reference with reduced variance, indicating faster and more stable convergence.
This supports our design: demographics-conditioned LoRA adapters align updates within patient subspaces, while the temporal hypernetwork prioritizes recent rounds without amplifying noise, thereby accelerating learning and mitigating heterogeneity-induced oscillations.
\paragraph{Impact of GMM}
The GMM-based demographics embedding allows the model to capture fine-grained patient heterogeneity via soft cluster assignments, producing \(\bm{\phi}_p\) that reflect subtle demographic and disease progression differences.
This yields higher CIDEr and ROUGE scores.
Figure~\ref{fig:patient_embeddings} shows a t-SNE projection of patient embeddings, colored by age and shaped by sex, exhibiting clear demographic structure and smooth age ordering.
Cluster densities also correlate with the five clinical sites, indicating that \(\bm{\phi}_p\) implicitly encodes both demographic and institutional variations without explicit supervision, which is beneficial for personalized federated adaptation.
\paragraph{Impact of Temporal Residual Aggregation Strategy} The temporal aggregation strategy via meta-learning dynamically balances contributions from different temporal checkpoints. By assigning adaptive weights to client updates from different time points, our model effectively mitigates temporal concept drift and selectively emphasizes critical clinical events over routine findings. 
\paragraph{Impact of Residual-based Update} Our residual-based update mechanism inherently stabilizes training under heterogeneous federated conditions. By ensuring each global update remains within a convex combination of previous states, the model prevents large disruptive updates, thereby maintaining stable convergence behavior and avoiding performance degradation often observed in traditional federated algorithms. This stability is empirically reflected by lower variance in validation loss across communication rounds, demonstrating robustness and rapid convergence.
\section{Conclusion and Discussion} \label{section:dis}
In this work, we introduced Federated Temporal Adaptation (FTA) as a new federated learning setting that explicitly models both cross-client heterogeneity and temporal evolution in longitudinal medical data. To address this setting, we proposed FedTAR, a parameter-efficient multimodal framework that couples demographics-aware LoRA personalization with meta-learned temporal residual aggregation to handle non-stationary patient trajectories. Experiments on large-scale, multi-institutional datasets show that FedTAR consistently outperforms strong FL baselines in terms of linguistic accuracy, temporal consistency, and convergence stability, while preserving data privacy. Beyond empirical gains, our analysis provides convergence guarantees under mild assumptions, suggesting a principled foundation for extending FTA to irregular visit patterns and more challenging heterogeneous regimes.


{
    \small

\begin{thebibliography}{40}
\providecommand{\natexlab}[1]{#1}
\providecommand{\url}[1]{\texttt{#1}}
\expandafter\ifx\csname urlstyle\endcsname\relax
  \providecommand{\doi}[1]{doi: #1}\else
  \providecommand{\doi}{doi: \begingroup \urlstyle{rm}\Url}\fi

\bibitem[Adnan et~al.(2022)Adnan, Kalra, Cresswell, Taylor, and Tizhoosh]{adnan2022federated}
Mohammed Adnan, Shivam Kalra, Jesse~C Cresswell, Graham~W Taylor, and Hamid~R Tizhoosh.
\newblock Federated learning and differential privacy for medical image analysis.
\newblock \emph{Scientific Reports}, 12\penalty0 (1):\penalty0 1953, 2022.

\bibitem[Atteia et~al.(2023)Atteia, El-kenawy, Samee, Jamjoom, Ibrahim, Abdelhamid, Azar, Khodadadi, Ghanem, and Shams]{atteia2023adaptive}
Ghada Atteia, El-Sayed~M El-kenawy, Nagwan~Abdel Samee, Mona~M Jamjoom, Abdelhameed Ibrahim, Abdelaziz~A Abdelhamid, Ahmad~Taher Azar, Nima Khodadadi, Reham~A Ghanem, and Mahmoud~Y Shams.
\newblock Adaptive dynamic dipper throated optimization for feature selection in medical data.
\newblock \emph{Computers, Materials \& Continua}, 75\penalty0 (1):\penalty0 1883--1900, 2023.

\bibitem[Brisimi et~al.(2018)Brisimi, Chen, Mela, Olshevsky, Paschalidis, and Shi]{brisimi2018federated}
Theodora~S Brisimi, Ruidi Chen, Theofanie Mela, Alex Olshevsky, Ioannis~Ch Paschalidis, and Wei Shi.
\newblock Federated learning of predictive models from federated electronic health records.
\newblock \emph{International Journal of Medical Informatics}, 112:\penalty0 59--67, 2018.

\bibitem[Chen and Pan(2024{\natexlab{a}})]{chen2024medical}
Jieying Chen and Rong Pan.
\newblock Medical report generation based on multimodal federated learning.
\newblock \emph{Computerized Medical Imaging and Graphics}, 113:\penalty0 102342, 2024{\natexlab{a}}.

\bibitem[Chen and Pan(2024{\natexlab{b}})]{chen2024federated}
X. Chen and Y. Pan.
\newblock Federated multimodal learning for medical report generation.
\newblock In \emph{Proceedings of International Conference on Medical Image Computing and Computer Assisted Intervention}, 2024{\natexlab{b}}.

\bibitem[Chen et~al.(2022)]{chen2022timefed}
Z. Chen et~al.
\newblock Time-weighted asynchronous federated learning for dynamic environments.
\newblock \emph{Sensors}, 22\penalty0 (4):\penalty0 1672, 2022.

\bibitem[Chiruvella et~al.(2021)Chiruvella, Guddati, et~al.]{chiruvella2021ethical}
Varsha Chiruvella, Achuta~Kumar Guddati, et~al.
\newblock Ethical issues in patient data ownership.
\newblock \emph{Interactive Journal of Medical Research}, 10\penalty0 (2):\penalty0 e22269, 2021.

\bibitem[Deng et~al.(2020)Deng, Kamani, and Mahdavi]{deng2020distributionally}
Yuyang Deng, Mohammad~Mahdi Kamani, and Mehrdad Mahdavi.
\newblock Distributionally robust federated averaging.
\newblock \emph{Proceedings of Advances in Neural Information Processing Systems}, 33:\penalty0 15111--15122, 2020.

\bibitem[et~al.(2019)]{johnson2019mimic}
Johnson Alistair~EW et al.
\newblock Mimic-cxr-jpg, a large publicly available database of labeled chest radiographs.
\newblock \emph{arXiv preprint arXiv:1901.07042}, 2019.

\bibitem[Fujita et~al.(2024)Fujita, Fushimi, Ito, Matsui, Tatsugami, Fujioka, Ueda, Fujima, Hirata, Tsuboyama, et~al.]{fujita2024advancing}
Shohei Fujita, Yasutaka Fushimi, Rintaro Ito, Yusuke Matsui, Fuminari Tatsugami, Tomoyuki Fujioka, Daiju Ueda, Noriyuki Fujima, Kenji Hirata, Takahiro Tsuboyama, et~al.
\newblock Advancing clinical mri exams with artificial intelligence: Japan’s contributions and future prospects.
\newblock \emph{Japanese Journal of Radiology}, pages 1--10, 2024.

\bibitem[Gama et~al.(2014)Gama, {\v{Z}}liobait{\.e}, Bifet, Pechenizkiy, and Bouchachia]{gama2014survey}
Jo{\~a}o Gama, Indr{\.e} {\v{Z}}liobait{\.e}, Albert Bifet, Mykola Pechenizkiy, and Abdelhamid Bouchachia.
\newblock A survey on concept drift adaptation.
\newblock \emph{ACM Computing Surveys}, 46\penalty0 (4):\penalty0 1--37, 2014.

\bibitem[Gerig et~al.(2016)Gerig, Fishbaugh, and Sadeghi]{gerig2016longitudinal}
Guido Gerig, James Fishbaugh, and Neda Sadeghi.
\newblock Longitudinal modeling of appearance and shape and its potential for clinical use, 2016.

\bibitem[Ghari and Shen(2024)]{ghari2024streamfed}
M. Ghari and S. Shen.
\newblock Streamfed: Online personalized federated learning with global model ensembles.
\newblock In \emph{Proceedings of Advances in Neural Information Processing Systems}, 2024.

\bibitem[Guan et~al.(2024)Guan, Yap, Bozoki, and Liu]{guan2024federated}
Hao Guan, Pew-Thian Yap, Andrea Bozoki, and Mingxia Liu.
\newblock Federated learning for medical image analysis: A survey.
\newblock \emph{Pattern Recognition}, page 110424, 2024.

\bibitem[Hamamci et~al.(2024)Hamamci, Er, and Menze]{hamamci2024ct2rep}
Ibrahim~Ethem Hamamci, Sezgin Er, and Bjoern Menze.
\newblock Ct2rep: Automated radiology report generation for 3d medical imaging.
\newblock In \emph{Proceedings of International Conference on Medical Image Computing and Computer-Assisted Intervention}, pages 476--486, 2024.

\bibitem[Hosny et~al.(2018)Hosny, Parmar, Quackenbush, Schwartz, and Aerts]{hosny2018artificial}
Ahmed Hosny, Chintan Parmar, John Quackenbush, Lawrence~H Schwartz, and Hugo~JWL Aerts.
\newblock Artificial intelligence in radiology.
\newblock \emph{Nature Reviews Cancer}, 18\penalty0 (8):\penalty0 500--510, 2018.

\bibitem[Jin et~al.(2021)Jin, Yu, Ke, Ding, Yi, Jiang, Duan, Tang, Chang, Wu, et~al.]{jin2021predicting}
Cheng Jin, Heng Yu, Jia Ke, Peirong Ding, Yongju Yi, Xiaofeng Jiang, Xin Duan, Jinghua Tang, Daniel~T Chang, Xiaojian Wu, et~al.
\newblock Predicting treatment response from longitudinal images using multi-task deep learning.
\newblock \emph{Nature Communications}, 12\penalty0 (1):\penalty0 1851, 2021.

\bibitem[Karimireddy et~al.(2020)Karimireddy, Kale, Mohri, Reddi, Stich, and Suresh]{karimireddy2020scaffold}
Sai~Praneeth Karimireddy, Satyen Kale, Mehryar Mohri, Sashank Reddi, Sebastian Stich, and Ananda~Theertha Suresh.
\newblock Scaffold: Stochastic controlled averaging for federated learning.
\newblock In \emph{Proceedings of International Conference on Machine Learning}, pages 5132--5143, 2020.

\bibitem[Kisilev et~al.(2015)Kisilev, Walach, Barkan, Ophir, Alpert, and Hashoul]{kisilev2015medical}
Pavel Kisilev, Eugene Walach, Ella Barkan, Boaz Ophir, Sharon Alpert, and Sharbell~Y Hashoul.
\newblock From medical image to automatic medical report generation.
\newblock \emph{IBM Journal of Research and Development}, 59\penalty0 (2/3):\penalty0 2--1, 2015.

\bibitem[Li et~al.(2021)Li, Wen, Wu, Hu, Wang, Li, Liu, and He]{li2021survey}
Qinbin Li, Zeyi Wen, Zhaomin Wu, Sixu Hu, Naibo Wang, Yuan Li, Xu Liu, and Bingsheng He.
\newblock A survey on federated learning systems: Vision, hype and reality for data privacy and protection.
\newblock \emph{IEEE Transactions on Knowledge and Data Engineering}, 35\penalty0 (4):\penalty0 3347--3366, 2021.

\bibitem[Li et~al.(2020)Li, Sahu, Zaheer, Sanjabi, Talwalkar, and Smith]{FedProx}
Tian Li, Anit~Kumar Sahu, Manzil Zaheer, Maziar Sanjabi, Ameet Talwalkar, and Virginia Smith.
\newblock Federated optimization in heterogeneous networks.
\newblock \emph{Proceedings of Machine Learning and Systems}, 2:\penalty0 429--450, 2020.

\bibitem[Lin(2004)]{lin2004rouge}
Chin-Yew Lin.
\newblock Rouge: A package for automatic evaluation of summaries.
\newblock In \emph{Text Summarization Branches Out}, pages 74--81, 2004.

\bibitem[Lu et~al.(2024)Lu, Pan, Dai, Si, and Zhang]{lu2024federated}
Zili Lu, Heng Pan, Yueyue Dai, Xueming Si, and Yan Zhang.
\newblock Federated learning with non-iid data: A survey.
\newblock \emph{IEEE Internet of Things Journal}, 2024.

\bibitem[McMahan et~al.(2017)McMahan, Moore, Ramage, Hampson, and y~Arcas]{FedAvg}
Brendan McMahan, Eider Moore, Daniel Ramage, Seth Hampson, and Blaise~Aguera y Arcas.
\newblock Communication-efficient learning of deep networks from decentralized data.
\newblock In \emph{Artificial Intelligence and Statistics}, pages 1273--1282, 2017.

\bibitem[Mei et~al.(2024)Mei, Mao, Cai, Yang, and Cambria]{mei2024medical}
Xin Mei, Rui Mao, Xiaoyan Cai, Libin Yang, and Erik Cambria.
\newblock Medical report generation via multimodal spatio-temporal fusion.
\newblock In \emph{Proceedings of the 32nd ACM International Conference on Multimedia}, pages 4699--4708, 2024.

\bibitem[Messina et~al.(2022)Messina, Pino, Parra, Soto, Besa, Uribe, And{\'\i}a, Tejos, Prieto, and Capurro]{messina2022survey}
Pablo Messina, Pablo Pino, Denis Parra, Alvaro Soto, Cecilia Besa, Sergio Uribe, Marcelo And{\'\i}a, Cristian Tejos, Claudia Prieto, and Daniel Capurro.
\newblock A survey on deep learning and explainability for automatic report generation from medical images.
\newblock \emph{ACM Computing Surveys}, 54\penalty0 (10s):\penalty0 1--40, 2022.

\bibitem[Metmer and Yang(2025)]{metmer2025fedmrg}
Hichem Metmer and Xiaoshan Yang.
\newblock Fedmrg: federated medical report generation via text-aware learning rate adjustment and multi-level prototype collaboration.
\newblock \emph{Multimedia Systems}, 31\penalty0 (2):\penalty0 170, 2025.

\bibitem[Newell~Jr et~al.(2004)Newell~Jr, Hogg, and Snider]{newell2004report}
JD Newell~Jr, JC Hogg, and GL Snider.
\newblock Report of a workshop: quantitative computed tomography scanning in longitudinal studies of emphysema.
\newblock \emph{European Respiratory Journal}, 23\penalty0 (5):\penalty0 769--775, 2004.

\bibitem[Papineni et~al.(2002)Papineni, Roukos, Ward, and Zhu]{papineni2002bleu}
Kishore Papineni, Salim Roukos, Todd Ward, and Wei-Jing Zhu.
\newblock Bleu: a method for automatic evaluation of machine translation.
\newblock In \emph{Proceedings of the 40th annual meeting of the Association for Computational Linguistics}, pages 311--318, 2002.

\bibitem[Rajpurkar et~al.(2022)Rajpurkar, Chen, Banerjee, and Topol]{rajpurkar2022ai}
Pranav Rajpurkar, Emma Chen, Oishi Banerjee, and Eric~J Topol.
\newblock Ai in health and medicine.
\newblock \emph{Nature Medicine}, 28\penalty0 (1):\penalty0 31--38, 2022.

\bibitem[Reddi et~al.(2020)Reddi, Charles, Zaheer, Garrett, Rush, Kone{\v{c}}n{\`y}, Kumar, and McMahan]{FedAdam}
Sashank Reddi, Zachary Charles, Manzil Zaheer, Zachary Garrett, Keith Rush, Jakub Kone{\v{c}}n{\`y}, Sanjiv Kumar, and H~Brendan McMahan.
\newblock Adaptive federated optimization.
\newblock \emph{arXiv preprint arXiv:2003.00295}, 2020.

\bibitem[Sanh et~al.(2019)Sanh, Debut, Chaumond, and Wolf]{sanh2019distilbert}
Victor Sanh, Lysandre Debut, Julien Chaumond, and Thomas Wolf.
\newblock Distilbert, a distilled version of bert: smaller, faster, cheaper and lighter.
\newblock In \emph{NeurIPS EMC\^2 Workshop}, 2019.

\bibitem[Sheller et~al.(2020)Sheller, Edwards, Reina, Martin, Pati, Kotrotsou, Milchenko, Xu, Marcus, Colen, et~al.]{sheller2020federated}
Micah~J Sheller, Brandon Edwards, G~Anthony Reina, Jason Martin, Sarthak Pati, Aikaterini Kotrotsou, Mikhail Milchenko, Weilin Xu, Daniel Marcus, Rivka~R Colen, et~al.
\newblock Federated learning in medicine: facilitating multi-institutional collaborations without sharing patient data.
\newblock \emph{Scientific Reports}, 10\penalty0 (1):\penalty0 12598, 2020.

\bibitem[Sloan et~al.(2024)Sloan, Clatworthy, Simpson, and Mirmehdi]{sloan2024automated}
Phillip Sloan, Philip Clatworthy, Edwin Simpson, and Majid Mirmehdi.
\newblock Automated radiology report generation: A review of recent advances.
\newblock \emph{IEEE Reviews in Biomedical Engineering}, 2024.

\bibitem[Smit et~al.(2020)Smit, Jain, Rajpurkar, Pareek, Ng, and Lungren]{chexbert}
Akshay Smit, Saahil Jain, Pranav Rajpurkar, Anuj Pareek, Andrew~Y. Ng, and Matthew~P. Lungren.
\newblock Chexbert: Combining automatic labelers and expert annotations for accurate radiology report labeling using bert, 2020.

\bibitem[Vedantam et~al.(2015)Vedantam, Lawrence~Zitnick, and Parikh]{vedantam2015CIDEr}
Ramakrishna Vedantam, C Lawrence~Zitnick, and Devi Parikh.
\newblock Cider: Consensus-based image description evaluation.
\newblock In \emph{Proceedings of the IEEE/CVF Conference on Computer Vision and Pattern Recognition}, pages 4566--4575, 2015.

\bibitem[Wang et~al.(2025{\natexlab{a}})]{wang2025fedmme}
L. Wang et~al.
\newblock Fedmme: Single-round federated multimodal ensemble for medical diagnosis.
\newblock \emph{arXiv preprint arXiv:2501.03292}, 2025{\natexlab{a}}.

\bibitem[Wang et~al.(2025{\natexlab{b}})Wang, Deng, Fan, Yin, and Ng]{wang2025multi}
Naibo Wang, Yuchen Deng, Shichen Fan, Jianwei Yin, and See-Kiong Ng.
\newblock Multi-modal one-shot federated ensemble learning for medical data with vision large language model.
\newblock \emph{arXiv preprint arXiv:2501.03292}, 2025{\natexlab{b}}.

\bibitem[Xie et~al.(2024)]{xie2024perada}
J. Xie et~al.
\newblock Perada: Parameter-efficient federated learning personalization with generalization guarantees.
\newblock In \emph{Proceedings of the IEEE/CVF Conference on Computer Vision and Pattern Recognition}, 2024.

\bibitem[Yi et~al.(2024)]{yi2024pfedlora}
Z. Yi et~al.
\newblock pfedlora: Personalized federated learning via low-rank adaptation for model heterogeneity.
\newblock In \emph{Proceedings of International Conference on Machine Learning}, 2024.

\end{thebibliography}

}
\clearpage
\setcounter{page}{1}
\maketitlesupplementary

\subsection{Proof of Theorem 1} \label{proof: Theorem1}
\subsection*{Notation and Setup}
Let \(w^{(t)}\in\mathbb{R}^d\) be defined recursively by
\begin{equation}\label{eq:recursion}
  w^{(t)} \;=\;(1-\alpha_t)\,w^{(t-1)} \;+\;\alpha_t\,\bar w^{(t)},\qquad
  t=1,2,\dots,T,
\end{equation}
with initial condition \(w^{(0)}\) given, and coefficients \(\alpha_t\in[0,1]\).  Define for \(0\le x\le t\):
\begin{align}
  \beta_0^{(t)} &:= \prod_{y=1}^t (1-\alpha_y), \label{eq:beta0}\\
  \beta_x^{(t)} &:= \alpha_x \;\prod_{y=x+1}^t (1-\alpha_y),\quad x=1,\dots,t. \label{eq:betax}
\end{align}

\subsection*{Lemma 1 (Normalization of Weights)}
\begin{lemma}\label{lem:normalization}
For each \(t\ge1\),
\[
  \beta_0^{(t)} \;+\;\sum_{x=1}^t\beta_x^{(t)} \;=\;1.
\]
\end{lemma}

\begin{proof}
We proceed by induction on \(t\).

\noindent\emph{Base case (\(t=1\)).}  From \eqref{eq:beta0} and \eqref{eq:betax}:
\[
  \beta_0^{(1)} = 1-\alpha_1,\quad
  \beta_1^{(1)} = \alpha_1,
\]
so \(\beta_0^{(1)} + \beta_1^{(1)} = (1-\alpha_1)+\alpha_1 =1.\)

\noindent\emph{Inductive step.}  Assume the identity holds for \(t-1\):
\(\sum_{x=0}^{t-1}\beta_x^{(t-1)}=1.\)  Then for \(t\):
\begin{align*}
  \sum_{x=0}^t\beta_x^{(t)}
  &= \beta_0^{(t)} + \sum_{x=1}^{t-1}\beta_x^{(t)} + \beta_t^{(t)}\\
  &= (1-\alpha_t)\prod_{y=1}^{t-1}(1-\alpha_y)
     + \sum_{x=1}^{t-1}\alpha_x\prod_{y=x+1}^{t}(1-\alpha_y)
     + \alpha_t\\
  &= (1-\alpha_t)\!\Bigl[\beta_0^{(t-1)} + \sum_{x=1}^{t-1}\beta_x^{(t-1)}\Bigr]
     + \alpha_t\\
  &= (1-\alpha_t)\cdot1 + \alpha_t = 1,
\end{align*}
where in the second line we used \(\prod_{y=x+1}^t(1-\alpha_y)
=\prod_{y=x+1}^{t-1}(1-\alpha_y)\,(1-\alpha_t)\).  
This completes the induction.
\end{proof}
\subsection*{Theorem 1 (Convex Combination of Past Snapshots)}
\begin{lemma}\label{thm:convex_hull_full}
Under recursion \eqref{eq:recursion}, for each \(t=1,\dots,T\),
\begin{equation}\label{eq:closed_form_full}
  w^{(t)}
  = \beta_0^{(t)}\,w^{(0)}
    \;+\;\sum_{x=1}^t \beta_x^{(t)}\,\bar w^{(x)},
\end{equation}
where the coefficients \(\{\beta_x^{(t)}\}_{x=0}^t\) are defined in \eqref{eq:beta0}–\eqref{eq:betax} and satisfy
\(\beta_x^{(t)}\ge0\) and \(\sum_{x=0}^t\beta_x^{(t)}=1\).  
Consequently, \(w^{(t)}\) lies in the convex hull of
\(\{\,w^{(0)},\,\bar w^{(1)},\dots,\bar w^{(t)}\}\).
\end{lemma}

\begin{proof}
We prove \eqref{eq:closed_form_full} by induction on \(t\).

\noindent\emph{Base case (\(t=1\)).}
Using \eqref{eq:recursion}:
\[
  w^{(1)}
  = (1-\alpha_1)\,w^{(0)} + \alpha_1\,\bar w^{(1)}
  = \beta_0^{(1)}\,w^{(0)} + \beta_1^{(1)}\,\bar w^{(1)}.
\]

\noindent\emph{Inductive step.}  Suppose \eqref{eq:closed_form_full} holds for \(t-1\):
\[
  w^{(t-1)}
  = \beta_0^{(t-1)}\,w^{(0)} + \sum_{x=1}^{t-1}\beta_x^{(t-1)}\,\bar w^{(x)}.
\]
Then by \eqref{eq:recursion},
\[
  w^{(t)}
  = (1-\alpha_t)\,w^{(t-1)} + \alpha_t\,\bar w^{(t)}.
\]
Substitute the inductive hypothesis:
\[
  w^{(t)}
  = (1-\alpha_t)\Bigl[\beta_0^{(t-1)}\,w^{(0)}
    + \sum_{x=1}^{t-1}\beta_x^{(t-1)}\,\bar w^{(x)}\Bigr]
    + \alpha_t\,\bar w^{(t)}.
\]
Distribute \((1-\alpha_t)\) and use definitions \eqref{eq:beta0}–\eqref{eq:betax}:
\begin{equation}
\begin{aligned}
    w^{(t)}
  = \underbrace{\Bigl[(1-\alpha_t)\beta_0^{(t-1)}\Bigr]}_{\beta_0^{(t)}}
    w^{(0)}
  &+ \sum_{x=1}^{t-1}\underbrace{\Bigl[(1-\alpha_t)\beta_x^{(t-1)}\Bigr]}_{\beta_x^{(t)}}
    \bar w^{(x)} \\
  &+ \underbrace{\alpha_t}_{\beta_t^{(t)}}\,\bar w^{(t)},
\end{aligned}    
\end{equation}
which is exactly \eqref{eq:closed_form_full} for \(t\).  Positivity of each \(\beta_x^{(t)}\) is immediate, and Lemma~\ref{lem:normalization} gives the partition-of-unity property.  Hence \(w^{(t)}\) is a convex combination of the stated points.
\end{proof}
Theorem~\ref{thm:convex_hull_full} shows that the residual update \eqref{eq:recursion} admits a closed‐form expression as a convex combination of the initial model and all intermediate aggregated models.  This both anchors the global state in the pretrained initialization \(w^{(0)}\) and ensures smooth temporal memory of past updates, with theoretical guarantees on stability and convergence rate inherited from convex‐combination geometry.
\subsection{Proof of Theorem\;\ref{thm:bounded_update}}
\label{sec:proof_bounded_update}

The only conditions required are the following.

\begin{assumption}[Residual boundedness]\label{asm:residual}
For each communication round $r\!\in\!\mathbb{N}$ and every time step $t\!\in\!\{1,\dots,T\}$,
\[
  \Delta_t \;=\; \bar{\bm w}_g^{(t,r)} - \bm w_g^{(t-1,r)}, 
  \qquad
  \|\Delta_t\| \;\le\; G,
  \quad
  0 < G < \infty .
\]
\end{assumption}

\begin{assumption}[Coefficient range]\label{asm:alpha}
The temporal aggregation coefficient satisfies
\(
  0 \le \alpha_t \le 1
\)
for every $t$.
\end{assumption}

\paragraph{Notation.}
All vectors lie in the same finite-dimensional Euclidean space $(\mathbb{R}^d,\|\cdot\|)$, though the argument is metric-agnostic; any norm may be used.  
Positive homogeneity of the norm,
$\|\lambda x\| = \lvert\lambda\rvert\,\|x\|$ for $\lambda\!\in\!\mathbb{R}$ and $x\!\in\!\mathbb{R}^d$,
is invoked throughout without further mention.

\medskip
\begin{lemma}[Exact increment expression]\label{lem:increment}
Let the global parameter at time step $t$ and round $r$ be updated by
\begin{equation}
  \bm w_g^{(t,r)}
  \;=\;
  \bm w_g^{(t-1,r)}
  \;+\;
  \alpha_t
  \bigl(
    \bar{\bm w}_g^{(t,r)} - \bm w_g^{(t-1,r)}
  \bigr).
  \tag{\ref{eq:residual_update}}
\end{equation}
Then for the same indices $(t,r)$ one has
\[
  \bm w_g^{(t,r)} - \bm w_g^{(t-1,r)}
  \;=\;
  \alpha_t\,\Delta_t.
\]
\end{lemma}

\begin{proof}
Equation~\eqref{eq:residual_update} is linear in
$\bm w_g^{(t-1,r)}$ and $\bar{\bm w}_g^{(t,r)}$.
Subtracting $\bm w_g^{(t-1,r)}$ from both sides isolates the increment:
\[
  \bm w_g^{(t,r)} - \bm w_g^{(t-1,r)}
  = \alpha_t\bigl(\bar{\bm w}_g^{(t,r)} - \bm w_g^{(t-1,r)}\bigr)
  = \alpha_t\,\Delta_t,
\]
since $\Delta_t$ is defined precisely as the bracketed difference.
\end{proof}

\begin{lemma}[Norm bound for a scaled residual]\label{lem:norm_bound}
Under Assumptions~\textup{\ref{asm:residual}}–\textup{\ref{asm:alpha}},  
\(
  \|
    \alpha_t \,\Delta_t
  \|
  \le
  \alpha_t\,G
  \quad
  \text{for every }t.
\)
\end{lemma}

\begin{proof}
Positive homogeneity gives
$
  \| \alpha_t \Delta_t \|
  = \alpha_t \| \Delta_t \|.
$
Assumption~\ref{asm:residual} substitutes the uniform bound $\|\Delta_t\|\le G$,
while Assumption~\ref{asm:alpha} ensures $\alpha_t\ge 0$, so
$
  \| \alpha_t \Delta_t \| \le \alpha_t G.
$
\end{proof}

\begin{theorem}[Bound on the global-update norm]\label{thm:bounded_update_restate}
Let $\{\bm w_g^{(t,r)}\}$ be generated by the residual rule \eqref{eq:residual_update}.  
Under Assumptions~\textup{\ref{asm:residual}}–\textup{\ref{asm:alpha}},
\begin{equation}\label{eq:bounded_increment}
  \bigl\|
    \bm w_g^{(t,r)} - \bm w_g^{(t-1,r)}
  \bigr\|
  \;\le\;
  \alpha_t\,G,
  \qquad\forall\,t,r.
\end{equation}
If, moreover, $\alpha_t \xrightarrow[t\to\infty]{} 0$, then
\begin{equation}\label{eq:vanishing_increment}
  \lim_{t\to\infty}
  \bigl\|
    \bm w_g^{(t,r)} - \bm w_g^{(t-1,r)}
  \bigr\|
  = 0,
  \qquad\forall\,r.
\end{equation}
\end{theorem}

\begin{proof}
Lemma~\ref{lem:increment} rewrites the increment as $\alpha_t\Delta_t$.  
Taking the norm of both sides and applying Lemma~\ref{lem:norm_bound} produces inequality~\eqref{eq:bounded_increment} immediately.

For the limit, fix an arbitrary round $r$.
Let $\varepsilon>0$ be given.
Because $\alpha_t\to 0$, there exists $T_\varepsilon$ such that
$\alpha_t\,G < \varepsilon$ whenever $t\ge T_\varepsilon$.
Combining this with~\eqref{eq:bounded_increment} yields
\[
  \bigl\|
    \bm w_g^{(t,r)} - \bm w_g^{(t-1,r)}
  \bigr\|
  \le
  \alpha_t\,G
  < \varepsilon,
  \qquad
  \forall\,t\ge T_\varepsilon.
\]
Since $\varepsilon$ was arbitrary, convergence in~\eqref{eq:vanishing_increment} follows by the definition of a limit.
\end{proof}

\subsection{Proof of Hypergradient Computation}
In this appendix we provide a complete, self-contained derivation of the
\emph{hypergradient}\/
\(
\nabla_{\psi}\mathcal{L}_{\mathrm{val}}\bigl(\bm w^{(T)}_g(\psi)\bigr)
\)
used in Algorithm~\ref{alg:global_temporal_aggregation}.
Throughout, the communication‐round index~$r$ is omitted for clarity.

\paragraph{Preliminaries and Notation}

Let the global model after aggregating the first \(t\in\{0,\dots,T\}\) time
points be denoted by
\(
\bm w_t \equiv \bm w^{(t)}_g\in\mathbb{R}^{d_w}
\)
with initial point
\(
\bm w_0=\bm w^{(0)}_g.
\)
For notational brevity set
\(
\Delta_t(\bm w_t)=\bar{\bm w}_g^{(t)}-\bm w_t,
\quad
\alpha_t(\psi)=
\mathrm{Softmax}\!\bigl[g(e(t);\psi)\bigr]_t
\in(0,1),
\)
so that\footnote{%
  All quantities depending on~$\psi$ are differentiable
  under the usual smoothness assumptions on \(g(\,\cdot\,;\psi)\).}
\begin{equation}\label{eq:update-rule-compact}
\bm w_{t+1}
=\bm w_{t}+\alpha_t(\psi)\,\Delta_t(\bm w_t)
\qquad(t=0,\dots,T-1).
\end{equation}
The outer-level objective is the held-out validation loss
\begin{equation}\label{eq:outer-loss}
\mathcal{L}_{\mathrm{val}}\bigl(\bm w_T\bigr)
\;=\;
\mathcal{L}_{\mathrm{val}}\bigl(\bm w_T(\psi)\bigr)
\;=\;
\ell_{\mathrm{val}}\!\bigl(f(\bm w_T),D_{\mathrm{val}}\bigr).
\end{equation}

\paragraph{Goal.}
Compute
\( \nabla_{\psi}\mathcal{L}_{\mathrm{val}}\bigl(\bm w_T(\psi)\bigr)\)
efficiently \emph{without} back-propagating through the entire
client-side training graph.

\paragraph{Forward Sensitivity Propagation}
Define the \emph{sensitivity matrix}
\(
\bm S_t
\;=\;
\frac{\partial\bm w_t}{\partial\psi}\in\mathbb{R}^{d_w\times d_\psi}.
\)
Differentiating \eqref{eq:update-rule-compact} w.r.t.\ \(\psi\) gives
\begin{align}
\bm S_{t+1}
&=
\bm S_t
\;+\;
\underbrace{\frac{\partial\alpha_t}{\partial\psi}}_{\displaystyle\bm a_t^{\!\top}}
        \,\Delta_t(\bm w_t)
\;+\;
\alpha_t
\underbrace{\frac{\partial\Delta_t}{\partial\bm w_t}}_{\displaystyle-\bm I_{d_w}}
        \bm S_t
\nonumber\\[2pt]
&=
(1-\alpha_t)\bm S_t
\;+\;
\Bigl(\nabla_{\!\psi}\alpha_t\Bigr)\,\Delta_t(\bm w_t)^{\!\top},
\label{eq:sensitivity-recursion}
\end{align}
where we used
\( \partial\Delta_t/\partial\bm w_t=-\bm I_{d_w} \)
because
\( \bar{\bm w}_g^{(t)} \)
is held fixed once local updates are complete.
The vector‐Jacobian product
\( \nabla_{\!\psi}\alpha_t \)
is given by
\(
\nabla_{\!\psi}\alpha_t
=
\alpha_t\,\Bigl(\bm I_{d_\psi}
-
\sum_{j=1}^{T}\alpha_j\nabla_{\!\psi}\log\alpha_j\Bigr)
\)
but can be obtained
\emph{implicitly} by automatic differentiation
of the scalar
\(
\alpha_t(\psi)
\)
in modern frameworks; no explicit Jacobian is needed.

Starting from
\( \bm S_0=\bm 0 \),
we iterate \eqref{eq:sensitivity-recursion} for
\(t=0,\dots,T-1\)
using the same time loop as the primal update
\eqref{eq:update-rule-compact}.
The final hypergradient follows by the chain rule:
\begin{equation}\label{eq:hypergradient-final}
\nabla_{\!\psi}\mathcal{L}_{\mathrm{val}}
=
\bm S_T^{\!\top}\,
\nabla_{\!\bm w}\mathcal{L}_{\mathrm{val}}(\bm w_T)
\;\in\mathbb{R}^{d_\psi}.
\end{equation}

\paragraph{Computational Complexity}

\begin{itemize}[leftmargin=18pt, itemsep=2pt]
  \item \textbf{Time.}  
        Each time step executes one extra
        vector‐Jacobian product and one
        \(d_w\times d_\psi\) matrix update,
        yielding an overall cost
        \(\mathcal{O}(T\cdot d_w\cdot d_\psi)\).
        In practice \(d_\psi\!\ll\!d_w\)
        (\eg, \(d_\psi\!=\!32\) for temporal MLP),
        so the overhead is negligible.
  \item \textbf{Memory.}  
        No intermediate \(\bm w_t\) needs to be stored:
        sensitivities \(\bm S_t\) can be updated \emph{in-place}
        because \eqref{eq:sensitivity-recursion} only
        references \(\bm S_t\) and
        the already‐available \(\Delta_t(\bm w_t)\).
        Thus the extra memory is
        \(\mathcal{O}(d_w d_\psi)\),
        independent of~\(T\).
  \item \textbf{Parallelism.}  
        The recursion shares the same loop
        as the primal update;
        both can be executed on-device (GPU) with minimal
        synchronization.
\end{itemize}
\paragraph{Connection to Implicit‐Function Differentiation}

Equation~\eqref{eq:sensitivity-recursion}
implements \emph{forward-mode} hypergradient propagation, avoiding
inversion of the Hessian
\( \nabla_{\!\bm w\bm w}^2 \mathcal{L}_{\mathrm{train}}\)
as required by classical implicit differentiation. %
For completeness,
if one chooses to merge all \(T\)~update steps into a single fixed-point
mapping
\( \mathcal{A}(\bm w,\psi)=\bm w \),
the implicit-function theorem
gives\footnote{%
  Provided
  \( \nabla_{\!\bm w}\mathcal{A} \) is non-singular.
  This holds in practice because the residual
  weights in \eqref{eq:update-rule-compact}
  satisfy \(0<\alpha_t<1\).}
\(
\nabla_{\!\psi}\bm w_T
=
-\bigl(\bm I-\nabla_{\!\bm w}\mathcal{A}\bigr)^{-1}
\nabla_{\!\psi}\mathcal{A},
\)
which recovers \eqref{eq:hypergradient-final}
when the Neumann–series inverse is unrolled forward
over time steps.
The forward formulation used here is therefore
numerically equivalent
but substantially simpler to implement.

The forward-mode sensitivity recursion
\eqref{eq:sensitivity-recursion} combined with
\eqref{eq:hypergradient-final}
yields an \(\mathcal{O}(T)\)‐time,
\(\mathcal{O}(d_w d_\psi)\)‐memory
hypergradient estimator
that is fully compatible with existing
federated training loops
and satisfies the convergence guarantees
established in
Theorems~\ref{thm:convex_hull}--\ref{thm:bounded_update}.
\subsection{Proof of Bilevel Optimization Procedure Admits}
We now prove that, under mild smoothness and Lipschitz conditions on both
the inner training loss and the outer meta-objective,
the bilevel procedure in Algorithm~\ref{alg:global_temporal_aggregation}
converges almost surely to a first-order stationary point.

\textbf{Notation and setup.}\;
Let $\bm\theta_r \!=\! \bm w_g^{(T,r)}\in\mathbb{R}^{d_{\text{model}}}$ be
the global model after the inner loop of round $r$,
and let $\psi_r\in\mathbb{R}^{d_\psi}$ collect the meta-parameters
governing the softmax coefficients
$\alpha_t(\psi_r)$.
The bilevel objective is
\begin{equation}
    \min_{\psi\in\Psi}
    \quad
    \mathcal{L}_{\mathrm{val}}\!\bigl(\bm\theta^\star(\psi),\psi\bigr),
    \qquad
    \text{s.t.}\;
    \bm\theta^\star(\psi)
    :=\arg\min_{\bm\theta}\,
      \mathcal{F}(\bm\theta,\psi),
    \label{eq:bilevel}
\end{equation}
with inner loss
$\mathcal{F}(\bm\theta,\psi)=\tfrac1K\!\sum_{k=1}^{K}
\mathcal{L}_{k}(\bm\theta,\psi;D_k)$.
Algorithm~\ref{alg:global_temporal_aggregation} (i) solves the inner
problem to machine precision through local fine-tuning plus residual
aggregation, then (ii) updates $\psi$ via a stochastic hypergradient step.

\medskip
\textbf{Assumption 1 (Smoothness \& Lipschitz properties).}
Throughout we assume:

\begin{enumerate}[label=\roman*), leftmargin=2em, itemsep=2pt]
\item For any fixed $\psi$, the map
      $\bm\theta\mapsto\mathcal{F}(\bm\theta,\psi)$ is
      continuously differentiable and $L_\theta$-smooth.
\item The validation loss
      $\mathcal{L}_{\mathrm{val}}(\bm\theta,\psi)$ is
      jointly $(L_\theta^{\mathrm{val}},L_\psi^{\mathrm{val}})$-smooth.
\item Each coefficient $\alpha_t(\psi)$,
      realized via $\alpha_t(\psi)=\operatorname{softmax}_t(g(e(t);\psi))$,
      is $L_\psi^{\alpha}$-Lipschitz.
\item The stochastic hypergradient estimator
      $\widehat g_r$ is unbiased and $\|\widehat g_r\|_2\le G$ a.s.
\end{enumerate}

These are standard for non-convex bilevel analysis and
are satisfied by modern transformer backbones with smooth
activations.

\medskip
\textbf{Exact hypergradient.}\;
Because we (approximately) solve the inner problem at each round,
the implicit-function theorem yields
\begin{equation}
\begin{aligned}
  \nabla_\psi\!
  \mathcal{L}_{\mathrm{val}}
  \bigl(\bm\theta^\star(\psi),\psi\bigr)
 & = \nabla_\psi
      \mathcal{L}_{\mathrm{val}}(\bm\theta,\psi)\\
    &- \nabla_{\bm\theta\psi}^2
      \mathcal{F}(\bm\theta,\psi)\,
      \bigl[\nabla_{\bm\theta\bm\theta}^2
            \mathcal{F}(\bm\theta,\psi)\bigr]^{-1}
      \nabla_{\bm\theta}
      \mathcal{L}_{\mathrm{val}}(\bm\theta,\psi),
  \label{eq:hypergrad}    
\end{aligned}
\end{equation}
evaluated at $\bm\theta=\bm\theta^\star(\psi)$.  The matrix inverse is
well-defined in a neighbourhood of any local minimiser
(cf.\ $L_\theta$-smoothness).

\medskip
\textbf{Lemma 1 (Descent per hypergradient step).}  
Let the step size be
$\eta_r = \eta_0/\sqrt{r}$.
Under Assumption 1,
\begin{equation}
    \begin{aligned}
        \mathbb{E}\!\bigl[
    \mathcal{L}_{\mathrm{val}}\bigl(\bm\theta^\star(\psi_{r+1}),\psi_{r+1}\bigr)
  \bigr]
  &\le
  \mathbb{E}\!\bigl[
    \mathcal{L}_{\mathrm{val}}\bigl(\bm\theta^\star(\psi_{r}),\psi_{r}\bigr)
  \bigr]\\
  &-\frac{\eta_r}{2}
     \mathbb{E}\!\bigl[
       \|\nabla_\psi
         \mathcal{L}_{\mathrm{val}}
         \bigl(\bm\theta^\star(\psi_{r}),\psi_{r}\bigr)\|_2^2
     \bigr]\\
  &+\frac{L_\psi^{\mathrm{val}}\!+\!(\eta_r L_\psi^{\mathrm{val}})^2}{2}\,G^2.
    \end{aligned}
\end{equation}
\emph{Proof sketch.}\;
Apply the $L_\psi^{\mathrm{val}}$-smoothness of
$\mathcal{L}_{\mathrm{val}}$,
take conditional expectation, then use
$\mathbb{E}[\widehat g_r]=\nabla_\psi\mathcal{L}_{\mathrm{val}}$.
\hfill\(\square\)

\medskip
\textbf{Theorem 1 (Almost-sure convergence to a stationary point).}\;
With $\eta_r=\eta_0/\sqrt{r}$ and the conditions of Assumption 1,
Algorithm~\ref{alg:global_temporal_aggregation} generates a sequence
$\{\psi_r\}_{r\ge0}$ satisfying
\[
  \lim_{R\to\infty}
    \mathbb{E}
      \Bigl[
        \tfrac1R \sum_{r=0}^{R-1}
        \bigl\|
          \nabla_\psi
          \mathcal{L}_{\mathrm{val}}
          \bigl(\bm\theta^\star(\psi_r),\psi_r\bigr)
        \bigr\|_2^2
      \Bigr]
  = 0,
\]
and $\psi_r$ converges almost surely to the set of first-order
stationary points of~\eqref{eq:bilevel}.

\emph{Proof.}\;
Sum the descent inequality of Lemma 1, telescope, then invoke
$\sum_r\eta_r=\infty$, $\sum_r\eta_r^2<\infty$,
and apply the Robbins–Siegmund lemma.
\hfill\(\square\)

\medskip
\textbf{Discussion.}\;
Theorem 1 guarantees that the meta-learner discovers a parameter
$\psi_\infty$ at which no descent direction exists.
Empirically (Sec.\,5), this equilibrium balances
\emph{adaptivity}—large $\alpha_t$ when data shift sharply—
and \emph{stability}—small $\alpha_t$ when updates are noisy—yielding
consistent gains in longitudinal report generation.

\section{Extra Experimental results}
\subsection{Per-institution Precision}
\begin{figure*}[t]
  \centering
  \begin{subfigure}[b]{0.8\textwidth}
    \includegraphics[width=\linewidth]{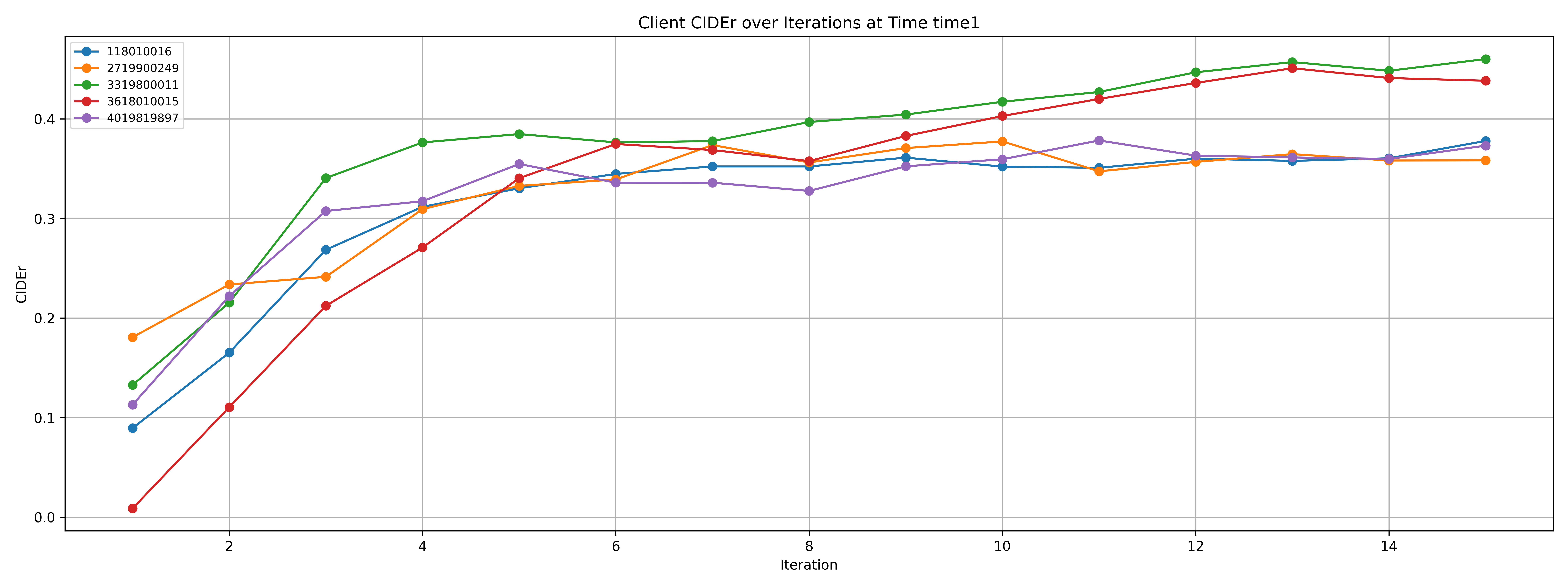}
    \caption{CIDEr scores of five clients at time 1.}
    \label{PIP-1}
  \end{subfigure}\hfill
  \begin{subfigure}[b]{0.8\textwidth}
    \includegraphics[width=\linewidth]{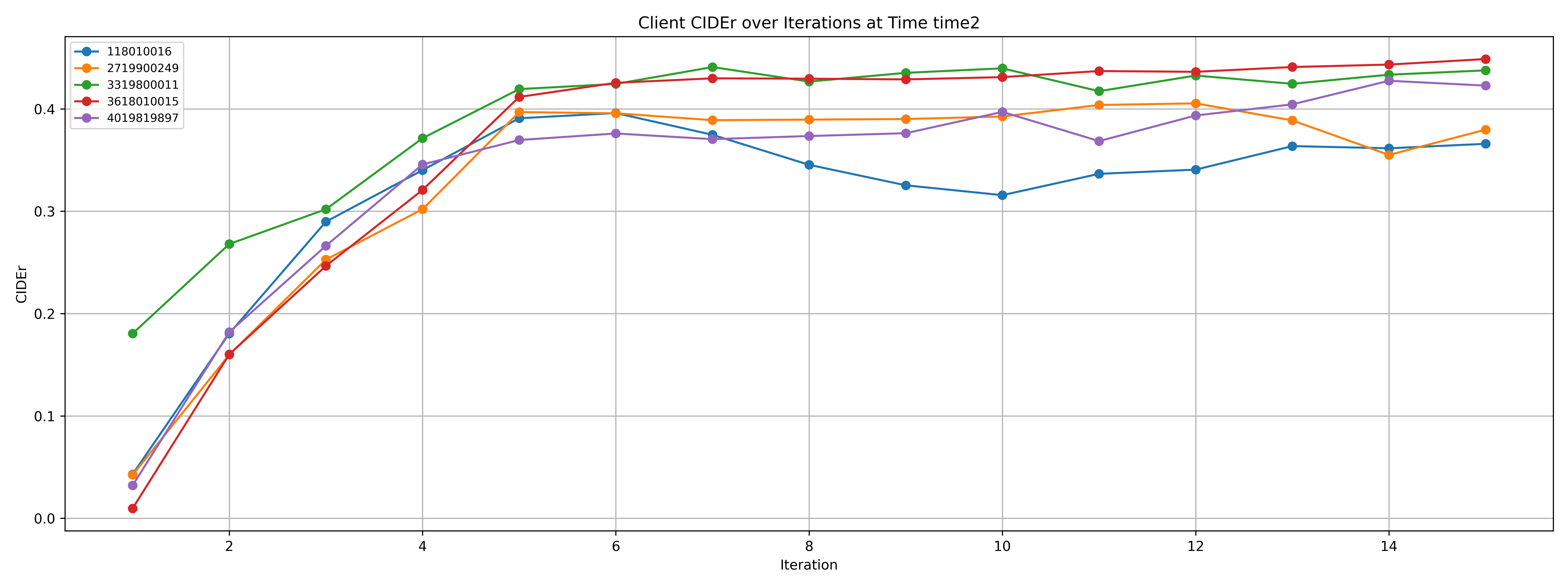}
    \caption{CIDEr scores of five clients at time 2.}
    \label{PIP-2}
  \end{subfigure}\hfill
  \begin{subfigure}[b]{0.8\textwidth}
    \includegraphics[width=\linewidth]{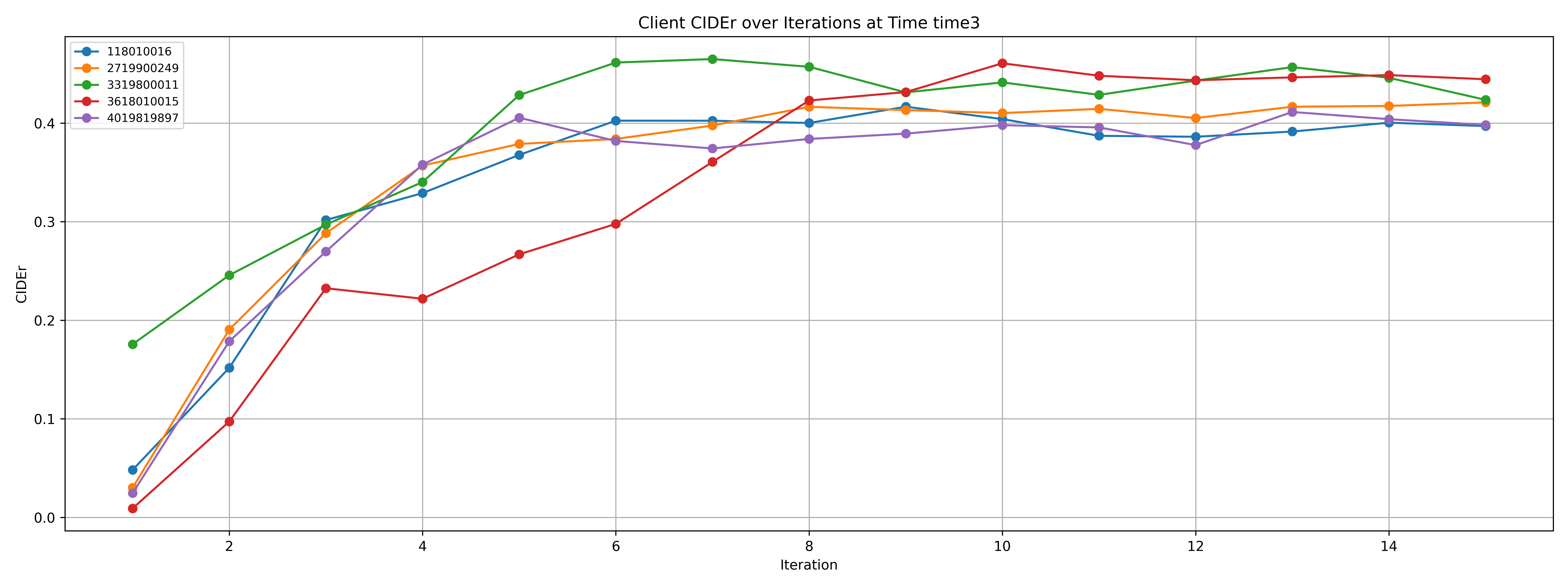}
    \caption{CIDEr scores of five clients at time 3.}
    \label{PIP-3}
  \end{subfigure}
\end{figure*}
\begin{figure*}[t]
  \centering
  \begin{subfigure}[b]{0.8\textwidth}
    \includegraphics[width=\linewidth]{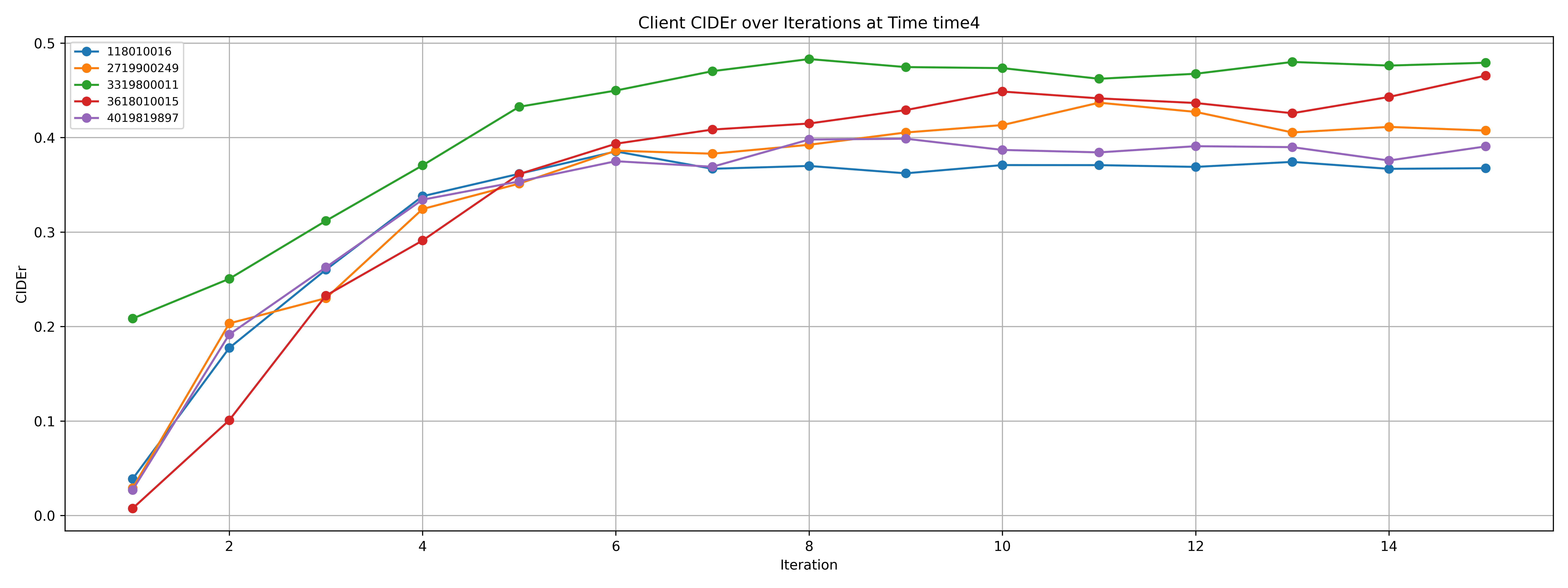}
    \caption{CIDEr scores of five clients at time 4.}
    \label{PIP-4}
  \end{subfigure}\hfill
  \begin{subfigure}[b]{0.8\textwidth}
    \includegraphics[width=\linewidth]{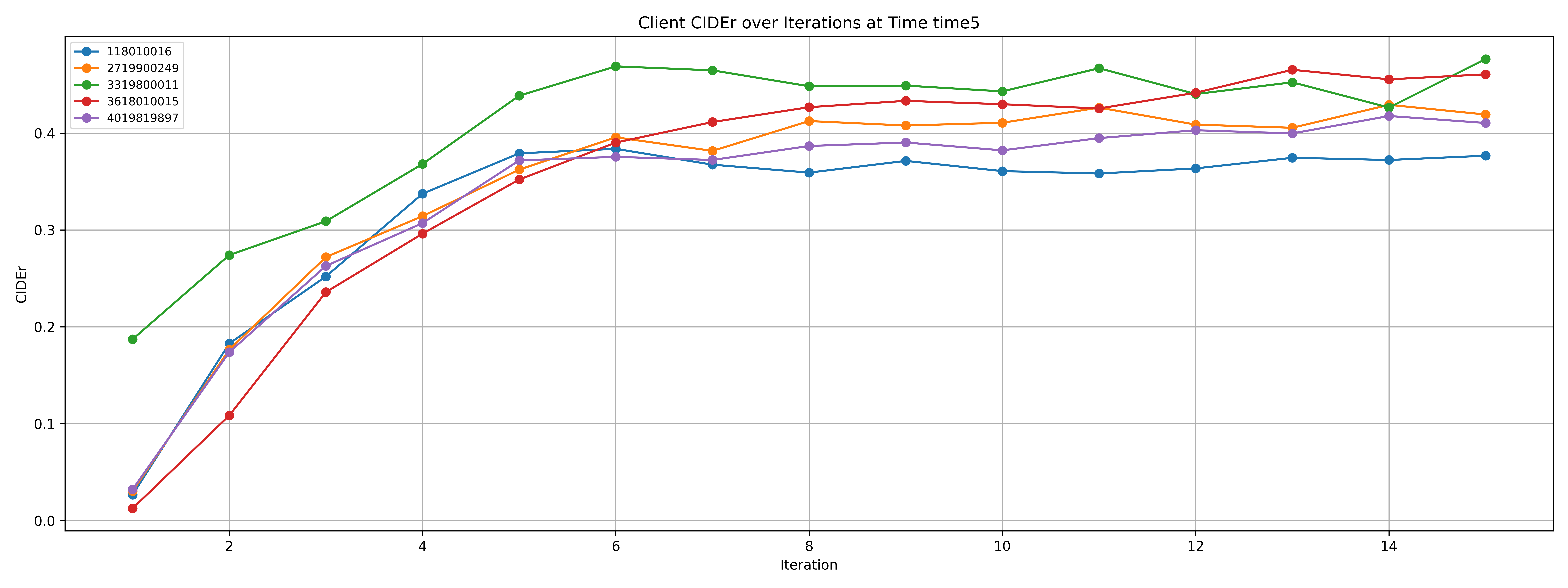}
    \caption{CIDEr scores of five clients at time 5.}
    \label{PIP-5}
  \end{subfigure}
\end{figure*}
\begin{figure*}[t]
    \centering
    \includegraphics[width=0.8\linewidth]{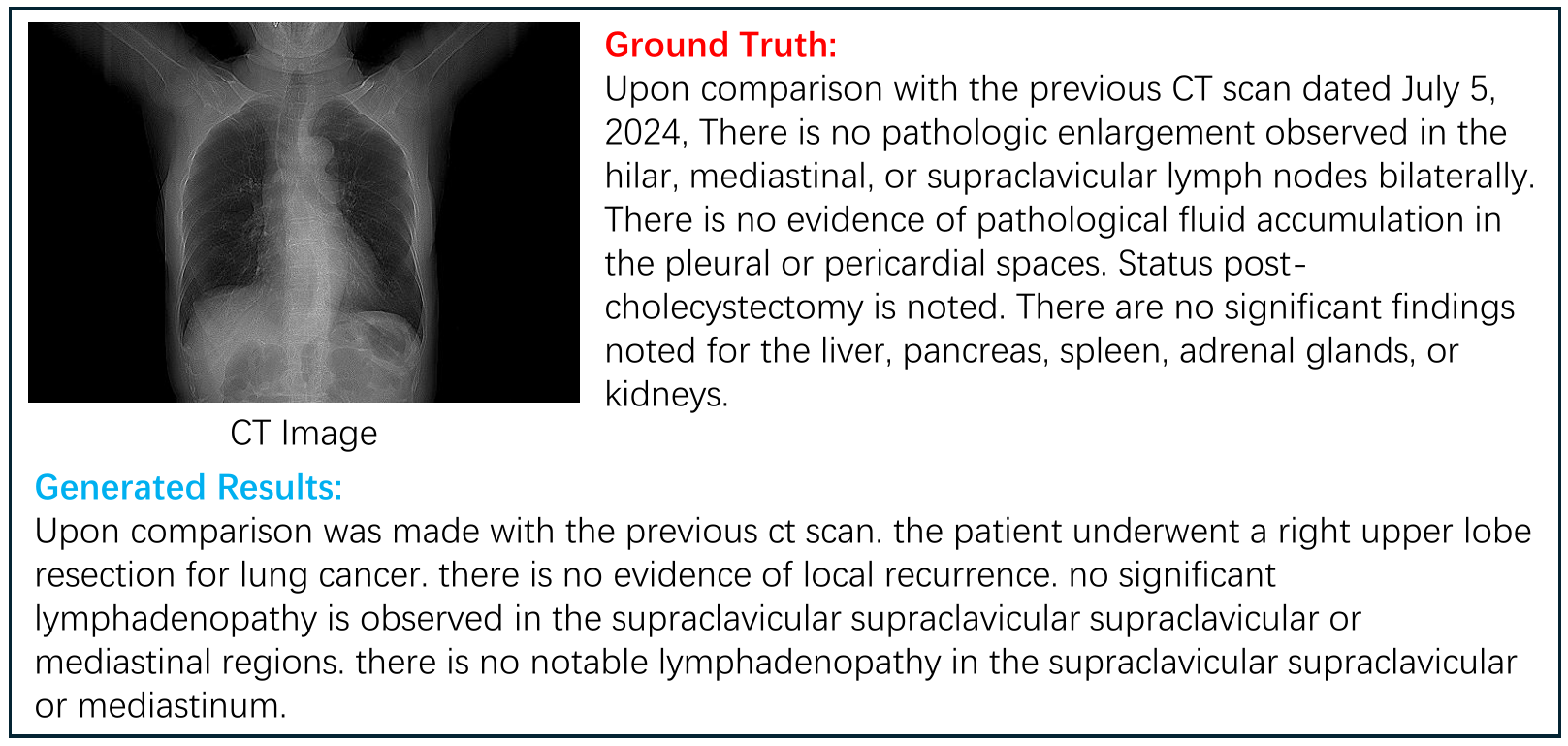}
    \caption{fig:generated report}
    \label{fig:generated report}
\end{figure*}
Figure~\ref{PIP-1}~\ref{PIP-2}~\ref{PIP-3}~\ref{PIP-4}~\ref{PIP-5} visualize per-client CIDEr trajectories at five successive time‐points (time1–time5). A few clear patterns emerge. First, all clients experience rapid gains in the earliest iterations, but the rate and ceiling of improvement differ markedly by client: for example, client 3319800011 (green) achieves the highest CIDEr—stabilizing above 0.45 by iteration 8—whereas clients 118010016 (blue) and 2719900249 (orange) plateau more modestly around 0.36. Second, the slow‐start client 3618010015 (red) benefits most from temporal weighting and personalization: its CIDEr rises from near zero at iteration 1 to exceed 0.45 by iteration 13, closing the gap with the best‐performing site. Third, the metadata‐conditioned LoRA adapters yield smoother convergence: fluctuation amplitudes shrink as time progresses (compare time1 vs. time5), indicating that the meta‐learned aggregation coefficients successfully dampen noisy updates. Finally, later time‐points not only improve overall CIDEr but also reduce inter‐client variance, demonstrating that our temporally‐aware federated adaptation both accelerates early learning and harmonizes performance across heterogeneous clinical sites.
\subsection{Generated Report}
The generated report in Fig.~\ref{fig:generated report} shows promising alignment with the core clinical message by correctly stating the absence of mediastinal or supraclavicular lymph-node enlargement and confirming no local recurrence—both critical findings for postoperative follow-up. Its concise style can speed up reading and demonstrates an emerging awareness of radiologic vocabulary such as “no evidence of.” However, several limitations remain: key details from the reference report (hilar nodes, effusion status, abdominal organs, prior cholecystectomy, and comparison date) are omitted, one surgical fact is hallucinated, and minor language issues—fragmented sentence structure and redundant word repetition—diminish professional polish. Overall, the draft delivers the essential negative findings but needs richer anatomical coverage, stricter factual fidelity, and cleaner syntax to match publication-grade radiology standards.
\section{Problem Formulation and Comparison}
\begin{table*}[t]
\centering
\small
\caption{Comparison of Federated Learning Paradigms.}
\label{tab:comparison}
\begin{tabular}{lccc}
\toprule
\textbf{Property} & \textbf{FedAvg} & \textbf{Online FL} & \textbf{FTA (Ours)} \\
\midrule
Data Dist. 
& $P_k$ (Static) & $Q_{k,\tau}$ (Streaming) & $P_{k,t}$ (Evolving) \\
Temporal Structure & Ignored & Optimization Time & Physical Time \\
Sequential Modeling & $\times$ & $\times$ & $\checkmark$ \\
Objective & Expectation & Regret / Stream & Longitudinal Sum \\
\bottomrule
\end{tabular}
\end{table*}
We formally define the Federated Temporal Adaptation (FTA) setting, which explicitly models the temporal evolution of client data distributions. Consider a federated system with $K$ clients. Unlike standard settings where data is viewed as a static collection, each client $k \in \{1, \dots, K\}$ holds a dataset $D_k$ consisting of sequential longitudinal data:
\begin{equation}
    D_k = \{(x_{k,t}, y_{k,t})\}_{t=1}^T,
\end{equation}
where each sample $(x_{k,t}, y_{k,t})$ is drawn from a time-dependent distribution $P_{k,t}$:
\begin{equation}
    (x_{k,t}, y_{k,t}) \sim P_{k,t}.
\end{equation}
Crucially, FTA posits that the underlying data distribution is \emph{non-stationary} and evolves over physical time steps $t$:
\begin{equation}
    P_{k,1} \neq P_{k,2} \neq \cdots \neq P_{k,T}.
\end{equation}
This formulation captures scenarios such as disease progression in longitudinal medical imaging, where the relationship between input $x$ and target $y$ shifts as a function of the patient's temporal state. The global optimization objective is to minimize the loss across all clients and all time steps, treating temporal drift as a first-class modeling component:
\begin{equation}
    \min_{w} \sum_{k=1}^K \sum_{t=1}^T L\big(f(w; x_{k,t}),\, y_{k,t}\big). \tag{FTA-Obj}
    \label{eq:fta}
\end{equation}

\paragraph{Comparison with FedAvg} In the canonical FedAvg setting, we assume that local data on client $k$ is drawn i.i.d. from a \emph{time-invariant} distribution $P_k$:
\begin{equation}
    (x, y) \sim P_k, \quad \forall (x,y) \in D_k.
\end{equation}
The key assumption here is stationarity: $P_{k,1} = P_{k,2} = \dots = P_{k,T} = P_k$. The optimization objective averages the expected risk over the static distributions:
\begin{equation}
    \min_{w} \sum_{k=1}^K \mathbb{E}_{(x,y)\sim P_k} L\big(f(w; x),\, y\big). \tag{FedAvg-Obj}
\end{equation}
Under this formulation, the temporal ordering of samples is ignored, and any temporal drift is treated as noise or distribution shift rather than a structured signal.
\paragraph{Comparison with Online FL} While Online FL involves sequential updates, it fundamentally differs from FTA in its definition of "time." In Online FL, the index $\tau$ refers to optimization rounds (or the arrival of a data stream) rather than the physical evolution of the underlying subject state. At round $\tau$, client $k$ receives a batch $B_{k,\tau}$ from a distribution $Q_{k,\tau}$:
\begin{equation}
    B_{k,\tau} \sim Q_{k,\tau}.
\end{equation}
The model is updated via a streaming rule, such as:
\begin{equation}
    w_{\tau+1} = w_\tau - \eta \nabla L(B_{k,\tau}). \tag{OnlineFL-Update}
\end{equation}
Here, $Q_{k,\tau}$ represents the distribution of the data stream at step $\tau$, which may be arbitrary or adversarial. In contrast, FTA's $P_{k,t}$ represents the \emph{physical temporal state} of the subject (e.g., a patient's condition in year $t$). Online FL focuses on regret minimization or adaptability to concept drift in the optimization landscape, whereas FTA focuses on modeling the structured longitudinal dependencies of the data itself.
\subsection{Summary of Differences}
We summarize the distinctions between the three paradigms in Table~\ref{tab:comparison}. FTA is the only framework that explicitly incorporates sequential modeling of physical time into the federated objective.
The FTA framework can be viewed as a generalization of the standard settings. 
\begin{itemize}
    \item \textbf{Relation to FedAvg:} If the distribution is stationary, i.e., $P_{k,t} \equiv P_k$ for all $t$, the FTA objective (Eq. \ref{eq:fta}) reduces to the standard FedAvg objective (scaled by $T$).
    \item \textbf{Relation to Online FL:} If we treat the physical time steps $t$ solely as indices for incoming data batches without enforcing longitudinal dependencies between $x_{k,t}$ and $x_{k,t-1}$, FTA resembles Online FL. However, FTA specifically leverages the sequential correlation (trajectory) inherent in $P_{k,t}$ for prediction.
\end{itemize}



\end{document}